\documentclass{article}

\usepackage{amsfonts,amsthm,amssymb}
\usepackage{bbm}
\usepackage{booktabs}
\usepackage{caption}
\usepackage{hyperref}
\usepackage{ifthen}
\usepackage{mathtools}
\usepackage{multicol}
\usepackage{subcaption}
\usepackage{tikz}
\usepackage{xcolor}
\usepackage{xurl}

\newcommand{\bbN}{\mathbb{N}}

\newcommand{\bbR}{\mathbb{R}}

\newcommand{\RR}{\mathbb{R}}

\newcommand{\bfa}{\mathbf{a}}
\newcommand{\bfb}{\mathbf{b}}
\newcommand{\bfe}{\mathbf{e}}

\newcommand{\bfv}{\mathbf{v}}
\newcommand{\bfx}{\mathbf{x}}
\newcommand{\bfy}{\mathbf{y}}
\newcommand{\bfz}{\mathbf{z}}

\newcommand{\bftheta}{\boldsymbol{\theta}}
\newcommand{\bfA}{\mathbf{A}}
\newcommand{\bfB}{\mathbf{B}}

\newcommand{\bfP}{\mathbf{P}}
\newcommand{\bfQ}{\mathbf{Q}}
\newcommand{\bfS}{\mathbf{S}}
\newcommand{\bfU}{\mathbf{U}}
\newcommand{\bfV}{\mathbf{V}}

\newcommand{\bfX}{\mathbf{X}}
\newcommand{\bfY}{\mathbf{Y}}
\newcommand{\bfSigma}{\boldsymbol{\Sigma}}
\newcommand{\bfone}{\mathbf{1}}
\newcommand{\bfnull}{\mathbf{0}}

\newcommand{\calA}{\mathcal{A}}

\newcommand{\calN}{\mathcal{N}}

\newcommand{\calP}{\mathcal{P}}

\newcommand{\calR}{\mathcal{R}}
\newcommand{\calS}{\mathcal{S}}

\newcommand{\rmc}{\mathrm{c}}
\newcommand{\rme}{\mathrm{e}}

\newcommand{\dd}{\,\mathrm{d}}

\newcommand{\nrows}{n}
\newcommand{\ncolumns}{d}
\newcommand{\ntemplates}{D}
\newcommand{\BirkhoffDimension}{N} 

\newcommand{\sort}[1]{{\downarrow}({#1})}

\newcommand{\SortEmbedding}[1]{\beta_{#1}}
\newcommand{\SortEmbeddingTwo}[2]{\delta_{#1,#2}}
\newcommand{\QuotientSortEmbedding}[1]{\overline \beta_{#1}}
\newcommand{\QuotientSortEmbeddingTwo}[2]{\overline \delta_{#1,#2}}

\newcommand{\GroupAction}[2]{{#1}{#2}}
\newcommand{\distance}[2]{\mathrm{dist}({#1},{#2})}
\newcommand{\CharacteristicFunction}[1]{K_{#1}}

\DeclarePairedDelimiter{\norm}{\lVert}{\rVert}
\DeclarePairedDelimiter{\abs}{\lvert}{\rvert}
\DeclarePairedDelimiterX{\set}[2]\lbrace\rbrace{#1 \,\delimsize\vert\, #2}

\newcommand{\Ran}{\operatorname{range}}

\newcommand{\linop}{L}
\newcommand{\linopmat}{\mathbf{L}}

\newtheorem{theorem}{Theorem}
\newtheorem{proposition}[theorem]{Proposition}
\newtheorem{corollary}[theorem]{Corollary}
\newtheorem{lemma}[theorem]{Lemma}

\theoremstyle{remark}
\newtheorem{remark}[theorem]{Remark}

\begin{document}

\title{Quantitative Bounds for Sorting-Based Permutation-Invariant Embeddings}

\author{
    Nadav~Dym\thanks{N.~Dym is with the Faculty of Mathematics, Technion-Israel Insitute of Technology, Technion City, Haifa, Israel. email: nadavdym@technion.ac.il},
    Matthias~Wellershoff\thanks{M.~Wellershoff was with the Department of Mathematics, University of Maryland, 4176 Campus Drive, College Park, MD 20742.},
    Efstratios~Tsoukanis\thanks{E.~Tsoukanis is with the Institute of Mathematical Sciences, Claremont Graduate University, 150 E.~10th Street, Claremont, CA 91711. email: efstratios.tsoukanis@cgu.edu}, \\
    Daniel~Levy\thanks{D.~Levy is with the Program in Applied and Computational Mathematics in Princeton University, Princeton, NJ 08540. Email: daniel.levy@princeton.edu}~
    and Radu~Balan\thanks{R.~Balan is with the Department of Mathematics in the University of Maryland, 4176 Campus Drive, College Park, MD 20742. email: rvbalan@umd.edu}
}

\maketitle


\begin{abstract}
    We study permutation-invariant embeddings of $d$-dimensional point sets, which are defined by sorting $D$ independent one-dimensional projections of the input. Such embeddings arise in graph deep learning where outputs should be invariant to permutations of graph nodes. Previous work showed that for large enough $D$ and projections in general position, this mapping is injective, and moreover satisfies a bi-Lipschitz condition. However, two gaps remain: firstly, the optimal size $D$ required for injectivity is not yet known, and secondly, no estimates of the bi-Lipschitz constants of the mapping are known. In this paper, we make substantial progress in addressing both of these gaps. Regarding the first gap, we improve upon the best known upper bounds for the embedding dimension $D$ necessary for injectivity, and also provide a lower bound on the minimal injectivity dimension. Regarding the second gap, we construct matrices of projection vectors, so that the bi-Lipschitz distortion of the mapping depends quadratically on the number of points $n$, and is completely independent of the dimension $d$. We also show that for any choice of projection vectors,  the distortion of the mapping will never be better than a bound proportional to the square root of $n$. Finally, we show that similar guarantees can be provided even when linear projections are applied to the mapping to reduce its dimension. 
\end{abstract}


\section{Introduction}
Consider the action of the symmetric group $S_\nrows$ on the matrices $\bbR^{\nrows \times \ncolumns}$ by row permutation, and let $\norm{\cdot}_F$ denote the Frobenius norm. We are interested in constructing functions $f : \bbR^{\nrows \times \ncolumns} \to \bbR^M$ that satisfy three main requirements:
    \begin{enumerate}
        \item \emph{Permutation invariance.} $f(\sigma \bfX) = f(\bfX)$ for all $\sigma \in S_\nrows$, $\bfX \in \bbR^{\nrows \times \ncolumns}$.
        \item \emph{Orbit separation.} $f(\bfX) = f(\bfY)$ implies $\bfX \in S_\nrows \bfY$ for all $\bfX, \bfY \in \bbR^{\nrows \times \ncolumns}$.
        \item \emph{Bi-Lipschitz condition} There exist constants $C_1, C_2 > 0$ such that, for all $\bfX, \bfY \in \bbR^{\nrows \times \ncolumns}$, 
        \begin{equation}\label{eq:biLipschitzcondition}
            C_1 \cdot \min_{\sigma \in S_\nrows} \norm{\bfX - \sigma \bfY}_\mathrm{F} \leq \norm{f(\bfX) - f(\bfY)}_2 
            \leq C_2 \cdot \min_{\sigma \in S_\nrows} \norm{\bfX - \sigma \bfY}_\mathrm{F}.
        \end{equation}
    \end{enumerate}

The motivation for these requirements comes from learning on multisets that is permutation-invariant. This is a common setting where one wishes to ``learn'' a permutation-invariant function $q(\bfX)$, using a parametric family of functions $f_\theta(\bfX)$ which is also permutation-invariant. A simple yet powerful and popular method to do this is the DeepSets model \cite{deepsets}. It applies a neural network $h_\theta$ to each of the rows $\bfx_i \in \bbR^d$ of $\bfX \in \bbR^{n \times d}$, and then  sums over all rows to obtain permutation invariance:
\begin{equation*}
    f_\theta(\bfX)=\sum_{j=1}^n h_\theta(\bfx_i).
\end{equation*}
It was shown in \cite{Amir2023Neural,pmlr-v237-tabaghi24a,Wang2024Polynomial,deepsets} that, if constructed correctly, the DeepSets model also has the orbit separation property. This orbit separation result guarantees that any permutation-invariant function can be approximated by a concatenation of a DeepSets model with an additional neural network \cite{wagstaff2022universal,deepsets} and is also used to provide maximally expressive graph neural networks \cite{NEURIPS2019_bb04af0f,xu2018how}. 

Recently, the bi-Lipschitz condition defined above has received more attention in the invariant learning community. The motivation for this requirement is controlling the quality of orbit separation, so that we can guarantee that orbits which are close to/far from each other are mapped to close/far vectors. Such properties can be useful, for example, for metric based learning tasks such as nearest neighbor search or clustering, as discussed in \cite{Cahill2024Towards}.  Unfortunately, the DeepSets model cannot be bi-Lipschitz \cite{Amir2023Neural}. Recent work suggests \cite{reshef2025noninjectivitypiecewiselinear}  that this is also the case for Janossy pooling: a generalization of DeepSets which sums over all $k$-tuples of rows of $\bfX$. These results inspired research to suggest new permutation-invariant functions which do have the bi-Lipschitz properties.

Among the most promising permutation-invariant functions that are bi-Lipschitz is the function proposed in \cite{Balan2025Permutation}, $\SortEmbedding{\bfA} : \bbR^{\nrows \times \ncolumns} \to \bbR^{\nrows \times \ntemplates} \simeq \bbR^{\nrows \ntemplates}$, defined as
    \begin{equation}\label{eq:sortembeddingone}
        \SortEmbedding{\bfA}(\bfX) := \begin{pmatrix}
            | & & | \\
            \sort{\bfX \bfa_1} & \dots & \sort{\bfX \bfa_\ntemplates} \\
            | & & |
        \end{pmatrix}, \qquad \bfX \in \bbR^{\nrows \times \ncolumns},
    \end{equation}
    where $\sort{\cdot} : \bbR^{\nrows} \to \bbR^{\nrows}$ denotes sorting vectors in a non-decreasing order and $(\bfa_k)_{k = 1}^\ntemplates \in \bbR^{\ncolumns}$ are the columns of $\bfA \in \bbR^{\ncolumns \times \ntemplates}$. It has been shown in \cite{Balan2025Permutation} that, for large enough $D$ and generic $\bfA$, this function is both orbit separating and bi-Lipschitz. The usefulness of this bi-Lipschitz mapping and the closely related FSW embedding \cite{Amir2024Fourier} for permutation-invariant learning tasks was demonstrated in \cite{Davidson2024Hoelder,sverdlov2024fsw}. In \cite{dym2025bi}, a variant of $\beta_\bfA$ is proposed which gives bi-Lipschitz invariants for the alternating group.  Other bi-Lipschitz permutation-invariant mappings include the max filter approach \cite{Cahill2024GroupInvariant} and group invariants based on coorbits \cite{Balan2023GInvariantI}. 

    To enable a theoretically informed choice between the different bi-Lipschitz permutation-invariant functions suggested in the literature, a more refined analysis is necessary. That is, a successful bi-Lipschitz invariant function $f$ should satisfy three additional requirements:
    \begin{enumerate}
        \setcounter{enumi}{3}
         \item \emph{Efficient computability.} $f$ can be computed in polynomial time with respect to $\nrows$ and $\ncolumns$, where, again, the lower the computational burden the better. 
        \item \emph{Small embedding dimension.} $M$ is as small as possible. It is known that necessarily $M\geq n\cdot d$ \cite{Joshi,Amir2023Neural} and so one would hope for $M$ to be as close to this lower bound as possible.
        \item \emph{Small distortion.} The distortion $C_2/C_1$ (where $C_1,C_2 > 0$ are the optimal constants satisfying equation~\eqref{eq:biLipschitzcondition}) is as close to one as possible.
    \end{enumerate}

The computational complexity of the function $\beta_{\bfA}$ is well understood. Our goal in this paper is to study the embedding dimension and distortion of the function $\beta_{\bfA}$, improving upon previous results obtained on this topic. We will now introduce some notation, and then review previous results, and give an overview of our main results. 

    \subsection{Notation} Our convention for the natural numbers is $\bbN = \lbrace 1,2,\dots \rbrace$. Given a natural number $n \in \bbN$, we denote $[n] := \lbrace 1,\dots,n \rbrace$. The cardinality (i.e., number of elements) of a finite set $S$ is denoted by $\abs{S}$. The complement of a subset $T \subset S$ is denoted by $T^\rmc := S \setminus T$. Additionally, we denote the characteristic function of $T$ by $K_T$, 
    \begin{equation*}
      x \in S \mapsto  \CharacteristicFunction{T}(x) := \begin{cases}
            1 & \mbox{if } x \in T, \\
            0 & \mbox{else}.
        \end{cases}
    \end{equation*}

    The $n$-dimensional vector of zeros is denoted by $\bfnull_n = (0~\dots~0) \in \bbR^n$ while the $n$-dimensional vector of ones is denoted by $\bfone_n = (1~\dots~1) \in \bbR^n$. Similarly, the $m \times n$ matrix of zeros is denoted by $\bfnull_{m \times n} \in \bbR^{m \times n}$. The two-norm of a vector $\bfx = (x_1~\dots~x_n) \in \bbR^n$ is
    \begin{equation*}
        \norm{\bfx}_2 = \left( \sum_{i = 1}^n x_i^2 \right)^{1/2}.
    \end{equation*}
    The unit sphere in $n$ dimensions is $S^{n-1} = \set{ \bfx \in \bbR^n }{\norm{\bfx}_2 = 1}$. The singular values of a matrix $\bfA \in \bbR^{m \times n}$ are denoted by $\sigma_1(\bfA), \dots, \sigma_{\min\lbrace m,n \rbrace}(\bfA)$ and assumed to be ordered non-increasingly; i.e., 
    \begin{equation*}
        \sigma_1(\bfA) \geq \dots \geq \sigma_{\min\lbrace m,n \rbrace}(\bfA).
    \end{equation*}
    The Frobenius norm of a matrix $\bfA \in \bbR^{m \times n}$ is
    \begin{equation*}
        \norm{\bfA}_\mathrm{F} =\left(\sum_{i=1}^m \sum_{j=1}^n A_{ij}^2\right)^{1/2}= \left( \sum_{i = 1}^{\min\lbrace m,n \rbrace} \sigma_1(\bfA)^2 \right)^{1/2}.
    \end{equation*}
    We say that a wide matrix $\bfA \in \bbR^{m \times n}$, $m \leq n$, is full spark (or has full spark) if every set of $m$ columns of $\bfA$ is linearly independent. Given an index set $I \subset [n]$ and a matrix $\bfA \in \bbR^{m \times n}$, we let $\bfA(I) \in \bbR^{m \times \abs{I}}$ be the matrix obtained from $\bfA$ by discarding all columns whose indices are not in $I$. We write $V \simeq W$ if two vector spaces, $V$ and $W$, are canonically isomorphic.
    
    If $f(x)$, $g(x)$ are two families of objects parametrized by $x \in S$, where $S$ is some set, then we write $f \lesssim g$ if there exists a constant $c > 0$ such that, for all $x \in S$, $f(x) \leq cg(x)$. We also write $f \gtrsim g$ if $g \lesssim f$. Similarly, when $f(n)$, $g(n)$ are parametrized by natural numbers $n \in \bbN$, we write $f(n) \in O(g(n))$ when $\limsup_{n \to \infty} \abs{f(n)/g(n)} < \infty$, $f(n) \in \Omega(g(n))$ when $\liminf_{n \to \infty} \abs{f(n)/g(n)} > 0$ and $f(n) \in \widetilde O(g(n))$ when there exists an $m \in \bbN$ such that $f(n) \in O(g(n) \log^m(n))$.

    Finally, we denote the group of permutations on $n$ elements by $S_n$. Elements of the group are denoted by $\sigma \in S_n$ or $\bfP \in S_n$ depending on whether we prefer to view them as permutations on $[n]$ or as matrices acting on $\bbR^n$.

    \subsection{Preliminaries and Roadmap}

    As mentioned before, we are interested in the action of the group $S_\nrows$ on $\bbR^{\nrows \times \ncolumns}$ by row permutation; or, more precisely, via 
    \begin{equation*}
        \sigma \bfX := \begin{pmatrix}
            \mbox{---} & \bfx_{\sigma(1)} & \mbox{---} \\
            & \vdots & \\
            \mbox{---} & \bfx_{\sigma(\nrows)} & \mbox{---}
        \end{pmatrix} \in \bbR^{\nrows \times \ncolumns},
    \end{equation*}
    where $\bfX \in \bbR^{\nrows \times \ncolumns}$ has rows $(\bfx_i)_{i = 1}^\nrows \in \bbR^{\ncolumns}$, and $\sigma \in S_\nrows$.  
    We write $\bfX \sim_{S_\nrows} \bfY$ if $\bfX = \sigma \bfY$ for some $\sigma \in S_\nrows$; equivalently, $\bfX \in S_\nrows \bfY$. The set of equivalence classes under this relation is denoted by $\bbR^{\nrows \times \ncolumns} / S_\nrows$ and carries a natural metric induced by the Frobenius norm:
    \begin{equation*}
        \distance{\bfX}{\bfY} := \min_{\sigma \in S_\nrows} \norm{\bfX - \sigma \bfY}_\mathrm{F}, \qquad \bfX, \bfY \in \bbR^{\nrows \times \ncolumns}.
    \end{equation*}

    Permutation-invariant functions $f : \bbR^{\nrows \times \ncolumns} \to \bbR^M$ descend to well-defined functions on the set of orbits $\bbR^{\nrows \times \ncolumns}/S_\nrows$. The sorting-based permutation-invariant embedding $\SortEmbedding{\bfA} : \bbR^{\nrows \times \ncolumns} \to \bbR^{\nrows \times \ntemplates}$, as defined in equation~\eqref{eq:sortembeddingone}, descends to $\QuotientSortEmbedding{\bfA} : \bbR^{\nrows \times \ncolumns}/S_\nrows \to \bbR^{\nrows \times \ntemplates}$. This insight allows us to reformulate orbit separation and the bi-Lipschitz condition of $\SortEmbedding{\bfA}$ simply as injectivity and bi-Lipschitz continuity of $\QuotientSortEmbedding{\bfA}$; the latter just being the condition
    \begin{equation*}
        C_1 \cdot \distance{\bfX}{\bfY} \leq \norm{\QuotientSortEmbedding{\bfA}(\bfX) - \QuotientSortEmbedding{\bfA}(\bfX)}_\mathrm{F} \leq C_2 \cdot \distance{\bfX}{\bfY},
    \end{equation*}
    for $\bfX, \bfY \in \bbR^{\nrows \times \ncolumns}/S_\nrows$. The optimal constants $C_1, C_2 > 0$ such that the above equation hold are called \emph{lower and upper Lipschitz constant} of $\QuotientSortEmbedding{\bfA}$. Their fraction $C_2/C_1$ is called \emph{distortion} of $\QuotientSortEmbedding{\bfA}$. 

    This paper grew out of \cite{Balan2025Permutation} and  \cite{Dym2024LowDimensional}.  The main result in \cite{Balan2025Permutation} states the following among other things.

    \begin{theorem}[{\cite[Theorem~1.2 on p.~3]{Balan2025Permutation}}]
    Let $d,n,D$ be natural numbers. 
    
        \hangindent\leftmargini
        \label{thm:BHS22}
        {\emph{1)}} For all $\bfA \in \bbR^{\ncolumns \times \ntemplates}$ such that $\QuotientSortEmbedding{\bfA}$ is injective, $\QuotientSortEmbedding{\bfA}$ is bi-Lipschitz continuous and the upper Lipschitz constant is given by the largest singular value $\sigma_1(\bfA)$.
        \begin{enumerate}
            \setcounter{enumi}{1}
            \item For $\ntemplates = \nrows!(\ncolumns-1)+1$ and all $\bfA \in \bbR^{\ncolumns \times \ntemplates}$ with full spark, $\QuotientSortEmbedding{\bfA}$ is bi-Lipschitz continuous with lower Lipschitz constant greater than or equal to 
            \begin{equation}\label{eq:sigma_d}
                \min_{\substack{I \subset [\ntemplates]\\\abs{I}=\ncolumns}} \sigma_\ncolumns(\bfA(I)).
            \end{equation}
            \item For all $\bfA \in \bbR^{\ncolumns \times \ntemplates}$ such that $\QuotientSortEmbedding{\bfA}$ is injective and almost all linear functions $\linop : \bbR^{\nrows \times \ntemplates} \to \bbR^{2 \nrows \ncolumns}$, the embedding 
                \begin{equation*}
                \bar\beta_{\bfA,\linop}:=\linop \circ \QuotientSortEmbedding{\bfA}
                \end{equation*}
            is bi-Lipschitz continuous.
        \end{enumerate}
    \end{theorem}

    It is noteworthy that in item~2) above the required number of templates $\ntemplates$ grows superexponentially in $n$. In particular, this scaling becomes prohibitive already for moderate values of $n$. As shown by one of the authors in earlier work, this dependence can be improved. More precisely, the factorial growth in $n$ can be replaced by a quadratic dependence, as stated in the following theorem.
      \begin{theorem}\label{thm:biLipschitzdeterministic}[From \cite{Ravina}]
      Let $d,r,n,D$ be natural numbers and let $\bfA\in \RR^{d\times D}$. If $\ntemplates \geq r\ncolumns((\nrows-1)^2+1)$, then the lower Lipschitz constant of $\smash{\QuotientSortEmbedding{\bfA}}$ is greater than or equal to  
        \begin{equation*}
            \min_{\substack{I \subset [\ntemplates]\\\abs{I} = r\ncolumns}} \sigma_\ncolumns(\bfA(I)).
        \end{equation*}
    \end{theorem}

    Taken together, the results above show the following. The map $\SortEmbedding{\bfA}$ is permutation-invariant, orbit separating, and satisfies the bi-Lipschitz condition~\eqref{eq:biLipschitzcondition} when $\bfA$ has full spark and the number of templates $\ntemplates$ scales quadratically with $\nrows$. In this setting, the embedding dimension equals $\nrows \ntemplates$, which scales cubically in $\nrows$.

    One may reduce this dimensionality by following item~3) of Theorem~\ref{thm:BHS22}, leading to a linear scaling in $\nrows$. However, this approach requires passing through an intermediate dimension $\nrows \ntemplates$ that scales superexponentially in $\nrows$, and is therefore impractical.

    An alternative linear projection strategy was proposed in \cite{Dym2024LowDimensional}. It yields a mapping $\SortEmbeddingTwo{\bfA}{\bfB} : \bbR^{\nrows \times \ncolumns} \to \bbR^{\ntemplates}$ defined by
    \[
        \SortEmbeddingTwo{\bfA}{\bfB}(\bfX)
        := \left( \bfb_k^\top \sort{\bfX \bfa_k} \right)_{k = 1}^\ntemplates,
        \qquad
        \bfX \in \bbR^{\nrows \times \ncolumns},
    \]
    where $\bfA \in \bbR^{\ncolumns \times \ntemplates}$ and $\bfB \in \bbR^{\nrows \times \ntemplates}$. Since $\SortEmbeddingTwo{\bfA}{\bfB}$ is permutation-invariant, it descends to a function $\QuotientSortEmbeddingTwo{\bfA}{\bfB} : \bbR^{\nrows \times \ncolumns}/S_\nrows \to \bbR^\ntemplates$. In \cite{Dym2024LowDimensional}, it was shown that this function is injective with embedding dimension $2nd+1$. Subsequently, \cite{Balan2024Stability} established that injectivity in this setting implies the bi-Lipschitz property. We summarize these results in the following theorem.

    \begin{theorem}[{\cite[Proposition~3.1 on p.~393]{Dym2024LowDimensional}} and \cite{Balan2024Stability}]\label{thm:DG24}
     Let $d,n,D$ be natural numbers.    If $\ntemplates \geq 2\nrows\ncolumns+1$, then $\QuotientSortEmbeddingTwo{\bfA}{\bfB}$ is injective for Lebesgue almost every $(\bfA,\bfB) \in \bbR^{\ncolumns \times \ntemplates} \times \bbR^{\nrows \times \ntemplates}$. Moreover, $\QuotientSortEmbeddingTwo{\bfA}{\bfB}$ is bi-Lipschitz continuous whenever it is injective. 
    \end{theorem}
    
This result provides an embedding dimension of $2nd+1$ for this new projection $\delta_{\bfA,\bfB}$, which is one more than the embedding dimension of $2nd$ required for $\beta_{\bfA,\linop}$ in Theorem \ref{thm:BHS22}. However, the advantage of this projection is that it is more efficient ($\bfB$ corresponds to a sparse $\linop$) and that this result only requires computing $\beta_\bfA$ with $D=2nd+1 $. In particular, $\beta_\bfA$ is orbit separating with this value, and so has a total embedding dimension of $M=Dn=2n^2d+n$. We note that it is known that any continuous, permutation-invariant injective function from $\RR^{n\times d}\to \RR^M$ must have $M\geq nd$ \cite{Joshi,Amir2023Neural}. Accordingly, the dimension for which we can ensure injectivity of $\bar \beta_{\bfA,\linop} $ and $\QuotientSortEmbeddingTwo{\bfA}{\bfB}$ are close to optimal, but a gap still remains. For $\bar \beta_\bfA$ there is a more substantial gap as the best embedding dimension we are currently aware of is quadratic in $n$. 

Another gap is that, while we know that all three mappings, $\bar \beta_\bfA, \bar \beta_{\bfA,\linop}$ and $\QuotientSortEmbeddingTwo{\bfA}{\bfB} $ are bi-Lipschitz whenever they are injective, we do not know much about their bi-Lipschitz distortion. We do know that the upper Lipschitz constant of $\bar \beta_{\bfA}$ is the first singular value of $\bfA$, and that the lower Lipschitz constant of  $\bar \beta_{\bfA}$ can be bounded by the expression in \eqref{eq:sigma_d}. However, this bound is not efficiently computable since it involves minimization over the minimal singular value of  $\binom{\ntemplates}{\ncolumns}$ different matrices. Moreover, we do not know if this bound is tight. And finally, we do not know how the bi-Lipschitz distortion depends on $n$ and $d$. Our aim in this paper is to address these issues. 

\subsection{Main Results}

\begin{table*}
    \centering
    \begin{tabular}{|c|c|c|c|}\hline
       & M &   Best upper bound for $M$ & Best lower bound for $M$\\\hline
    $ \QuotientSortEmbedding{\bfA} $ & nD & $ n^2(d-1)+n$ (see Thm.~\ref{thm:orbit_separation_beta})  & $\Omega(d\cdot n \log(n))$ (see Thm.~\ref{thm:matousek2})  \\\hline
   $ \QuotientSortEmbeddingTwo{\bfA}{\bfB}$ & D  & $
   (2n-1)d$ (see Thm.~\ref{thm:dimensionreduction})  & $nd$ (see \cite{Joshi}) \\\hline
    $\QuotientSortEmbedding{\bfA,\linop}$  & M  & $(2n-1)d$  (see Thm.~\ref{thm:dimensionreduction2})  & $nd $ (see \cite{Joshi})\\ \hline
    \end{tabular}
    \caption{Summary of the best known upper and lower bounds on the dimension $M$ needed for injectivity. The lower bound is understood as a necessary condition for $M$. The upper bound represents a sufficient condition that insures that generically the corresponding map is injective.}
    \label{tab:injectivity}
\end{table*}

The key findings of this paper are summarized below.
\begin{enumerate}
\item Building on known results from \cite{matousek} for the case $d=2$, we show that for $D\geq n(d-1)+1 $ the mapping $\QuotientSortEmbedding{\bfA}$ will be injective as long as $\bfA$ is full spark (see Theorem~\ref{thm:orbit_separation_beta} in Section~II-A). Conversely, we show that the lowest possible $D$ for which injectivity is possible is at best proportional to $(d-1)\cdot \log(n)$ (see Theorem~\ref{thm:matousek2} in Section~\ref{ssec:OrbitSeparation}). As a result, the embedding dimension $D\cdot n $ of $\beta_{\bfA}$ cannot be better than $\Omega(d\cdot n\log(n))$ (see Theorem~\ref{thm:matousek2} in Section~\ref{ssec:OrbitSeparation}).
\item We show that $\bar \beta_{\bfA,\linop}$ and $\QuotientSortEmbeddingTwo{\bfA}{\bfB}$ are injective with an embedding dimension of $(2n-1)d$ (see Theorem~\ref{thm:dimensionreduction} for $\QuotientSortEmbeddingTwo{\bfA}{\bfB}$ and Theorem~\ref{thm:dimensionreduction2} for $\bar \beta_{\bfA,\linop}$ in Section~\ref{ssec:OrbitSeparationB}).
\item Numerical experiments for small parameters $d > 1$ and $n > 2$, based on \cite[Proposition~3.8 on p.~14]{Balan2025Permutation}, show that our results are, typically, suboptimal\footnote{For $n = 2$, our results are optimal.}; by which we mean that there exist $D < n(d-1)+1$ and $\bfA \in \bbR^{d \times D}$ such that $\SortEmbedding{\bfA}$ separates orbits (see the discussion following Theorem~\ref{thm:orbit_separation_beta} in Section~\ref{ssec:OrbitSeparation}).
\end{enumerate}

Following these results, we summarize the best known upper and lower bounds for the injectivity of $\QuotientSortEmbedding{\bfA}$, $\bar \beta_{\bfA,\linop}$ and $\QuotientSortEmbeddingTwo{\bfA}{\bfB}$ in Table \ref{tab:injectivity}. Our next results pertain to bi-Lipschitz distortion: 
\begin{enumerate}\addtocounter{enumi}{3}

\item We show that the distortion of $\bar{\beta}_{\bfA}$ cannot be better than a bound proportional to $\sqrt{n}$ (see Theorem~\ref{thm:LowerBoundDistortion} in Section~\ref{ssec:LowerBoundDistortion}). 

\item We give two probabilistic constructions (and an explicit construction for $d=2$) of $\bfA$, such that $\QuotientSortEmbedding{\bfA}$ achieves a bi-Lipschitz distortion which scales like $n^2$, but is independent of $d$ (see Section~\ref{par:FirstConstruction} and Theorems~\ref{thm:biLipschitzprobabilistic} as well as \ref{thm:DistortionUnitSphere} in Sections~\ref{par:SecondConstruction} and \ref{par:ThirdConstruction}). These results require $D$ to be on the order of $n^2d$ and $n^4d$, respectively. 

\item Using a sketching argument, we show that $\QuotientSortEmbedding{\bfA,\linop}$ with an embedding dimension proportional to $nd$, up to logarithmic terms, can achieve  similar bi-Lipschitz distortion to $\QuotientSortEmbedding{\bfA}$ (see Theorem~\ref{th-betaA,S} in Section~\ref{ssec:BiLipschitzBoundsBetaAL}). 
\end{enumerate}
 
\subsection{Related work}
\label{subsec:related_sliced_wasserstein}

\paragraph{Wasserstein distance} The constructions studied in this paper admit a natural interpretation in terms of Wasserstein and sliced Wasserstein distances. In particular, the embedding $\beta_{\mathbf A}$ can be viewed as a finite-dimensional, Monte--Carlo approximation of the sliced $2$-Wasserstein distance between empirical measures. We briefly recall the relevant definitions.

Let $\mu$ and $\nu$ be probability measures on a metric space $(X,d_X)$. The \emph{$p$-Wasserstein distance} is defined by
\begin{equation*}
    \mathrm{W}_p(\mu,\nu)
    :=
    \left(
        \inf_{\gamma \in \Pi(\mu,\nu)}
        \int_{X \times X} d_X(x,y)^p \, \dd \gamma(x,y)
    \right)^{1/p},
\end{equation*}
where $\Pi(\mu,\nu)$ denotes the set of transport plans having marginals $\mu$ and $\nu$. When $\mu$ and $\nu$ are uniform empirical measures supported on $n$ points in $\bbR^d$,
\begin{equation*}
    \mu = \frac{1}{n} \sum_{i=1}^n \delta_{\bfx_i},
    \qquad
    \nu = \frac{1}{n} \sum_{i=1}^n \delta_{\bfy_i},
\end{equation*}
with $\bfX,\bfY \in \bbR^{n\times d}$ containing $(\bfx_i)_{i=1}^n$ and $(\bfy_i)_{i=1}^n$ as rows, respectively, the $2$-Wasserstein distance admits the explicit representation
\begin{equation*}
    \mathrm{W}_2(\mu,\nu)^2
    =
    \min_{\sigma \in S_n}
    \frac{1}{n}
    \norm{\bfX - \sigma \bfY}_{\mathrm F}^2
    =
    \frac{1}{n} \operatorname{dist}(\bfX,\bfY)^2.
\end{equation*}
For general $d$, this minimization can be solved in $O(n^3)$ time using the Hungarian method~\cite{Kuhn1955}. In the special case $d=1$, however, the optimal permutation is obtained by sorting, yielding an $O(n\log n)$ algorithm.

Motivated by the computational simplicity of the one-dimensional case, \cite{Rabin2011WassersteinBA} introduced what is called the sliced $p$-Wasserstein distance, defined for Borel probability measures on $\bbR^d$ by
\begin{equation*}
    \mathrm{SW}_p(\mu,\nu)
    :=
    \left(
        \int_{S^{d-1}}
        \mathrm{W}_p\bigl(
            (\operatorname{proj}_{\bftheta})_\ast \mu,
            (\operatorname{proj}_{\bftheta})_\ast \nu
        \bigr)^p
        \, \dd \bftheta
    \right)^{1/p},
\end{equation*}
where $\operatorname{proj}_{\bftheta}\bfx = \bftheta^\top \bfx$ and $(\operatorname{proj}_{\bftheta})_\ast$ denotes the pushforward. In practice, this integral is approximated by Monte--Carlo sampling:
\begin{align*}
    \mathrm{SW}_p(\mu,\nu)^p
   & \approx
    \frac{1}{D}
    \sum_{k=1}^D
    \mathrm{W}_p\bigl(
        (\operatorname{proj}_{\bftheta_k})_\ast \mu,
        (\operatorname{proj}_{\bftheta_k})_\ast \nu
    \bigr)^p\\
&    =:
    \widetilde{\mathrm{SW}}_p(\mu,\nu;(\bftheta_k)_{k=1}^D)^p,
\end{align*}
where $(\bftheta_k)_{k=1}^D \subset S^{d-1}$ are sampled independently, for instance uniformly. When $\mu$ and $\nu$ are uniform empirical measures as above, one obtains
\begin{equation*}
    \mathrm{SW}_2(\mu,\nu)^2
    =
    \int_{S^{d-1}}
    \frac{1}{n}
    \norm{\sort{\bfX \bftheta} - \sort{\bfY \bftheta}}_2^2
    \, \dd \bftheta,
\end{equation*}
and therefore the sampled sliced distance satisfies
\begin{align*}
    \widetilde{\mathrm{SW}}_2(\mu,\nu;(\bftheta_k)_{k=1}^D)^2
    &=
    \frac{1}{nD}
    \sum_{k=1}^D
    \norm{\sort{\bfX \bftheta_k} - \sort{\bfY \bftheta_k}}_2^2
 \\
 &   =
    \frac{1}{nD}
    \norm{
        \beta_{\boldsymbol{\Theta}}(\bfX)
        -
        \beta_{\boldsymbol{\Theta}}(\bfY)
    }_{\mathrm F}^2,
\end{align*}
where $\boldsymbol{\Theta} \in \bbR^{d\times D}$ contains the directions $(\bftheta_k)$ as columns. Thus, the embedding $\beta_A$ with columns sampled from the unit sphere corresponds exactly to a finite-dimensional approximation of the sliced $2$-Wasserstein distance between empirical measures.

The distortion bounds established in Theorem~\ref{thm:DistortionUnitSphere} therefore translate directly into quantitative comparisons between the Wasserstein distance and its Monte--Carlo sliced approximation (see Corollary~\ref{cor:WassersteinPositive}). In particular, for $D \gtrsim dn^2\log(n\sqrt d + \log n)$, the sampled sliced distance provides, with high probability, a bi-Lipschitz approximation of $\mathrm{W}_2$ on empirical measures with support size $n$, with distortion of order $\widetilde O(n^2)$. Similarly, Theorem~\ref{thm:LowerBoundDistortion} provides a converse result in that the distortion cannot be better than a bound proportional to $\sqrt{n}$.

For $d=2$, distortion bounds of order $O(n^2)$ were obtained in~\cite{Carriere2017Sliced}. For higher dimensions, previously known bounds were substantially weaker~\cite{Weighill2023Coarse}. Our probabilistic constructions provide $O(n^2)$ distortion for all $d$. On the other hand, for measures with infinite support, bi-Lipschitz equivalence between Wasserstein and sliced Wasserstein distances is impossible~\cite{Bayraktar2021Strong}, although H\"older-type bounds are available~\cite{Bonnotte2013Unidimensional}.

\paragraph{Further Related Work: Max Filter Banks, Sorted Coorbits, and Rotation Groups} The max filter construction was introduced in \cite{Cahill2024GroupInvariant} and further expanded in \cite{Mixon2023Max,Mixon2025Injectivity,qaddura2025max}. 
For the problem considered here, the {\em max filter} associates to a {\em template} $\mathbf W \in\bbR^{n\times d}$ the function $\mathbf X\in\bbR^{n\times d}\mapsto f_{\mathbf W}(\mathbf X)= \max_{\sigma\in S_n}\operatorname{trace}(\sigma \mathbf W \mathbf X^T)$. The aforementioned works prove that $M=2nd+1$ generic templates $\mathbf W_1,\ldots,\mathbf W_M$ in $\bbR^{n\times d}$ produce a bi-Lipschitz orbit separating embedding $\mathbf X \mapsto F(\mathbf X)=(f_{\mathbf W_1}(\mathbf X),\ldots,f_{\mathbf W_{M}}(\mathbf X))\in\bbR^M$. 

For sorted coorbits, the approaches introduced in \cite{Balan2025Permutation} and \cite{Cahill2024GroupInvariant} have been unified and generalized in \cite{Balan2023GInvariantII}. In subsequent works \cite{Balan2023GInvariantI,Balan2024Stability}, this construction has been shown to provide bi-Lipschitz embeddings. On the other hand, \cite{Cahill2024Towards} shows that smooth G-invariant embeddings for finite groups cannot be bi-Lipschitz.
  
For rotation groups, it was shown in  \cite{Derksen2024BiLipschitz} that the square root of the Gram matrix yields a bi-Lipschitz rotation invariant mapping. In \cite{Amir2026Stability}, bi-Lipschitzness of the square root is discussed with respect to arbitrary unitary actions on generic low dimensional domains.  

    \section{Estimating Embedding Dimensions}\label{sec:TheoreticalResults}

    \subsection{Embedding Dimension of \texorpdfstring{$\beta_{\bfA}$}{beta}}
    \label{ssec:OrbitSeparation}

    We first show that the embedding $\SortEmbedding{\bfA}$ separates orbits for full spark matrices $\bfA \in \bbR^{\ncolumns \times \ntemplates}$ with $\ntemplates > \nrows(\ncolumns-1)$ scaling like a linear polynomial in $\nrows$ and $\ncolumns$. Thereby, we  improve on Theorem~\ref{thm:biLipschitzdeterministic} which required matrices with $\ntemplates \geq r\ncolumns((\nrows-1)^2+1)$ scaling linearly in $\ncolumns$ but quadratically in $\nrows$. Secondly, we show that there is a lower bound on $D$ (depending on $d$ and $n$) below which $\beta_\bfA$ cannot separate orbits. Finally, we improve on the results of \cite{Dym2024LowDimensional} which imply injectivity of $\bar \beta_\bfA$ for generic $\bfA$ with embedding dimension of  $D=2nd+1$.

    \begin{theorem}\label{thm:orbit_separation_beta}
        Let $n,D$ and $d>1$ be natural numbers, and let $\bfA \in \bbR^{\ncolumns \times \ntemplates}$ be a full spark matrix. If $\ntemplates \geq \nrows(\ncolumns-1)+1$, then $\QuotientSortEmbedding{\bfA}$ is injective.
    \end{theorem}
    
    \begin{proof}
        Given fixed $d,D\in \bbN$, consider the minimal $\nrows \in \bbN$ such that $\QuotientSortEmbedding{\bfA} : \bbR^{\nrows \times \ncolumns}/S_\nrows \to \bbR^{\nrows \times \ntemplates}$ is not injective. Then, there exist $\bfX,\bfY \in \bbR^{\nrows \times \ncolumns}$ such that $\bfX \not\sim_{S_n} \bfY$ and $\SortEmbedding{\bfA}(\bfX) = \SortEmbedding{\bfA}(\bfY)$. By the minimality of $\nrows$, no row $\bfx_i$ of $\bfX$ equals a row $\bfy_j$ of $\bfY$ (since we could otherwise delete those rows to contradict minimality).

        For each pair of rows $(\bfx_i,\bfy_j)$, consider the columns $\bfa_k$ of $\bfA$ which are perpendicular to $\bfx_i - \bfy_j$,
        \begin{equation*}
            I_{i,j} := \set{k \in [\ntemplates]}{ \bfx_i - \bfy_j \perp \bfa_k }, \qquad i,j \in [\nrows].
        \end{equation*}
        Since $\bfx_i \neq \bfy_j$ and $\bfA$ has full spark, $\abs{ I_{i,j} } \leq \ncolumns-1$.

        For each row $\bfx_i$ and each column $\bfa_k$, there must be a row $\bfy_j$ such that $k \in I_{i,j}$ because $\SortEmbedding{\bfA}(\bfX) = \SortEmbedding{\bfA}(\bfY)$. Therefore, $[\ntemplates] \subset \bigcup_{j = 1}^\nrows I_{i,j}$ which implies 
        \begin{equation*}
            \ntemplates \leq \sum_{j=1}^n |I_{i,j}|\leq  \nrows(\ncolumns-1).
        \end{equation*}
        The theorem is thus proven.
    \end{proof}

    Next, we obtain a lower bound on the embedding dimension.
    
    \begin{theorem}\label{thm:matousek2}
        Let $n,D$ and $d>1$ be natural numbers such that $\lceil D/(d-1) \rceil \leq \log_2(n)+1$. Then, for any $\bfA\in \mathbb{R}^{d\times D}$, the map $\QuotientSortEmbedding{\bfA}$ is \emph{not} injective. Equivalently, if the map $\QuotientSortEmbedding{\bfA}$ is injective then $D=m(d-1)-r$ for $m,r$ integers with $0\leq r\leq d-2$ and $m>\log_2(n)+1$.
    \end{theorem}
    
    \begin{proof}
        Let $\bfA$ be any matrix in $\mathbb{R}^{d\times D}$ and denote its columns by $\bfa_1,\dots,\bfa_D$. For $k=\lceil D / (d-1) \rceil$, we have $k(d-1) \geq D $. Thus, we can partition $[D]$ into $k$ different sets $J_1,\dots,J_k$ which are all of cardinality strictly less than $d$. For each set $J_j$, choose some nonzero vector $\bfv_j$ which is orthogonal to all $\bfa_i$, $i\in J_j$.  For a choice of real numbers $\alpha_1,\dots,\alpha_k$ and $I \subset [k]$, denote
        \begin{equation*}
            \bfv(I) := \sum_{i\in I} \alpha_i \bfv_i,
        \end{equation*}
        where $\bfv(I)$ is the zero vector when $I$ is the empty set. 
      We choose the $\alpha_i$ so that $\bfv(I) \neq 0$ for all $I$ with $|I|$ odd. 
        Lebesgue almost every choice of $\alpha_i$ fulfills this requirement.

        Now, let $\bfX$ be a matrix whose rows are all vectors $\bfv(I)$ with $|I|$ even, and let $\bfY$ be a matrix whose rows are all vectors $\bfv(I)$ with $|I|$ odd. The number of rows of $\bfX$ and $\bfY$ is the same, $n = 2^k/2=2^{k-1} $. By assumption all rows of $\bfY$ are non-zero, while $\bfX$ contains an all-zero row (corresponding to the empty set). Therefore, $\bfX$ and $\bfY$ are not related by a permutation. For every $i=1,\dots,D$, we have that $\bfa_i$ is in some $J_j$, and so is orthogonal to $\bfv_j $. It follows that, for all $I\subseteq [k]$, 
        \begin{equation*}
            \bfa_i^\top \bfv(I) = \bfa_i^\top \bfv(I \triangle \{ j\}),
        \end{equation*}
        where $\triangle$ denotes the symmetric difference. Since the map $I \mapsto I \triangle \{j\} $ is a bijection for the index sets of even cardinality to the index sets of odd cardinality, we deduce that
        \begin{equation*}
            {\downarrow}\begin{pmatrix}
                \bfa_i^\top \bfx_1 \\
                \vdots \\
                \bfa_i^\top \bfx_n
            \end{pmatrix} = {\downarrow}\begin{pmatrix}
                \bfa_i^\top \bfy_1 \\
                \vdots \\
                \bfa_i^\top \bfy_n
            \end{pmatrix},
        \end{equation*}
        where $(\bfx_i)_{i = 1}^n$, $(\bfy_i)_{i = 1}^n$ denote the rows of $\bfX$ and $\bfY$, respectively. Since this is true for all $i$, we see that $\QuotientSortEmbedding{\bfA}$ is not injective when $n \geq 2^{k-1}$, which is equivalent to
        \begin{equation*}
            \left\lceil \frac{D}{d-1} \right\rceil = k \leq \log_2(n)+1.
        \end{equation*}
        The theorem is thus proven.
    \end{proof}

    \begin{remark}
        The two theorems above, Theorem \ref{thm:orbit_separation_beta} and Theorem \ref{thm:matousek2}, are stated in \cite{matousek} for the case $d=2$ and using different but equivalent notation. Our contribution here is in extending the proof to the general case $d\geq 2$. Also, in \cite{matousek} it is shown that, for $d=2$, the logarithmic lower bound is nearly attainable: there exist constants $D_0$ and $c$ such that, for all generic matrices with $D \geq D_0$ rows, the mapping $\bar \beta_{\bfA}$ is injective whenever $n \leq 2^{cD/\log D}$; or, equivalently, when $\log_2(n) \lesssim D / \log D$. It remains unclear whether similar bounds hold when $d > 2$.
    \end{remark}

    We now visualize the preceding two results. According to Theorem~\ref{thm:orbit_separation_beta}, the map $\SortEmbedding{\bfA}$ separates orbits for full spark matrices $\bfA \in \mathbb{R}^{d \times D}$ with $d>1$ whenever $D \geq n(d-1)+1$. In contrast, Theorem~\ref{thm:matousek2} shows that $\SortEmbedding{\bfA}$ fails to separate orbits (independently of the choice of $\bfA$) whenever $\lceil \ntemplates / (d-1) \rceil \leq \log_2(n)+1$. For $n=2$, these bounds coincide and yield the sharp threshold $D \ge 2d-1$ for orbit separation (assuming $\bfA$ has full spark). Via the connection to real phase retrieval established in~\cite{Balan2023Relationships}, this recovers the classical result (see, e.g., \cite{Balan2006On}) that $2d-1$ measurements are necessary and sufficient for sign retrieval in $\mathbb{R}^d$ provided the measurement vectors form a full spark frame.
    
    For $n>2$, however, the upper and lower bounds no longer match. In general, it remains open whether either bound is sharp in this regime, although we present some tentative progress in this direction in Section~\ref{sec:NumericalResults}. Figure~\ref{fig:Gaps} illustrates the resulting gap between the non orbit separating region (on or below the lower curve) and the orbit separating region for full spark matrices (on or above the upper line). The gap widens as $n$ increases and becomes more pronounced for larger $d$.

    \begin{figure}
    	\centering
    	\begin{tikzpicture}[scale=0.9]
            \draw[->] (1.7,0.85) -- (10.3,0.85) node[above] {$n$};
            \draw[->] (1.7,0.85) -- (1.7,5.65) node[right] {$D$};
        
            \foreach \n in {2,...,10} {
                \draw (\n,0.8) -- (\n,0.9) node[below] {\n};
            }
            \foreach \D in {2,...,11} {
                \draw (1.65,{\D/2}) -- (1.75,{\D/2}) node[left] {\D};
                \ifthenelse{\D > 3}{\draw[dotted, help lines] (1.7,{\D/2}) -- (10,{\D/2});}{}
            }
            \draw[dotted, help lines] (1.7,1.5) -- (7,1.5);
            \draw[dotted, help lines] (1.7,1) -- (3,1);
        
            \draw[dotted, domain=1.7:10.3] plot (\x, {(\x + 1)/2});
            \draw[dotted, domain=1.7:10.3] plot (\x, {(ln(\x)/ln(2)+1)/2});
            \node at (7.5,5) {$D = n+1$};
            \node at (8.9,1.5) {$D = \log_2(n)+1$};
        
            \fill[gray!20, opacity=0.5]
                plot[domain=1.7:10.3] (\x, {(\x + 1)/2}) --
                plot[domain=10.3:1.7] (\x, {(ln(\x)/ln(2)+1)/2}) --
                cycle;
            
            \foreach \n in {2,...,10} {
                \filldraw[black] (\n, {(\n + 1)/2}) circle (1pt);
                \filldraw[black] (\n, {floor(ln(\n)/ln(2)+1)/2}) circle (1pt);
                \draw[dotted, help lines] (\n,0.85) -- (\n, {(\n + 1)/2});

                \pgfmathsetmacro{\Dmin}{floor(ln(\n)/ln(2)+1)+1}
                \pgfmathsetmacro{\Dmax}{\n}

                \foreach \D in {\Dmin,...,\Dmax} {
                    \ifthenelse{\n > 2}{\fill[gray!80] (\n, {\D/2}) circle (1pt);}{}
                }
            }

            \filldraw[black] (4,2) circle (1pt);
            \filldraw[black] (5,2.5) circle (1pt);
        \end{tikzpicture}
    	\caption{
        Phase diagram for orbit separation when $d=2$. On and below the lower curve, $\SortEmbedding{\bfA}$ does not separate orbits for any choice of $\bfA \in \mathbb{R}^{2 \times D}$. On and above the upper line, $\SortEmbedding{\bfA}$ separates orbits provided that $\bfA$ has full spark. For larger $d$, the qualitative behavior remains the same, but the vertical scale changes by a factor of $(d-1)$. Black dots indicate parameter pairs $(n,D)$ for which it is known whether there exists a matrix $\bfA$ such that $\SortEmbedding{\bfA}$ separates orbits. Gray dots mark parameter pairs for which this question remains open. The two black dots at $(n,D)=(4,4)$ and $(5,5)$ correspond to cases where an explicit construction of such a matrix $\bfA$ is known (see Section~\ref{sec:NumericalResults}).}
    	\label{fig:Gaps}
    \end{figure}

\subsection{Embedding Dimension for \texorpdfstring{$\beta_{\bfA,\linop}$ and $\delta_{\bfA,\bfB}$}{projections} }
\label{ssec:OrbitSeparationB}

Theorem \ref{thm:matousek2} gives us a lower bound on the dimension $D$ for which $\beta_{\bfA}$ can be orbit separating which is proportional to $d\log_2(n)$. Since the output of $\beta_{\bfA}$ is $nD$ dimensional, the embedding dimension is in $\Omega(d\cdot n \log(n))$ at best. A better embedding dimension can be obtained by $\SortEmbeddingTwo{\bfA}{\bfB}$ and $\SortEmbedding{\bfA,\linop} $. In the following result, we show that $\SortEmbeddingTwo{\bfA}{\bfB}$ separates orbits for generic matrices $\bfA \in \bbR^{\ncolumns \times \ntemplates}$ and $\bfB \in \bbR^{\nrows \times \ntemplates}$ with $\ntemplates \geq (2\nrows-1)\ncolumns$. We thereby improve Theorem~\ref{thm:DG24} in which $\ntemplates \geq 2\nrows\ncolumns+1$ is required. We will then show a similar result for $\SortEmbedding{\bfA,\linop} $. 

    Our approach is based on, and improves upon, the proof of Theorem~\ref{thm:DG24} from \cite{Dym2024LowDimensional}. At a high level, the proof of Theorem~\ref{thm:DG24} uses dimension bounds and basic algebraic geometry to prove injectivity on the $nd$-dimensional domain $\RR^{n\times d}$, as long as $D\geq 2nd+1$. Our improvement is based on the observation that, to prove injectivity on $\RR^{n\times d}$, it is sufficient to prove injectivity on a lower-dimensional set. This observation utilizes the invariants of the action of $S_\nrows$ on $\bbR^{\nrows \times \ncolumns}$ (which are precisely the $\nrows \times \ncolumns$ matrices with constant columns). 
     
      For the proof, we will use some basic real algebraic terminology such as semi-algebraic sets and semi-algebraic functions. We recall the definition of these in Appendix \ref{app:semialgebraic}. We also use the following result from \cite{Dym2024LowDimensional} (cf.~also Amir et al.~\cite[Theorem~A.1 on p.~13]{Amir2023Neural}).

    \begin{theorem}[{Finite witness theorem; reformulation of \cite[Theorem~2.7 on p.~387]{Dym2024LowDimensional}}]
        Let $s,p$ be natural numbers, and let $\calS$ be a semialgebraic set of dimension $s$, let $f : \calS \times \bbR^p \to \bbR$ be a semialgebraic function and define the set
        \begin{equation*}
            \calN := \set{ x \in \calS }{ \forall \bftheta \in \bbR^p ,~ f(x,\bftheta) = 0 }.
        \end{equation*}
        If
        \begin{equation*}
            \dim \set{ \bftheta \in \bbR^p }{ f(x,\bftheta) = 0 } < p, \qquad \mbox{for all } x \in \calS \setminus \calN,
        \end{equation*}
        then there exists a semialgebraic set $\calR \subset \bbR^{p \times (s+1)}$ of dimension (strictly) less than $p(s+1)$ such that, for all $(\bftheta_1~\dots~\bftheta_{s+1}) \not\in \calR$,
        \begin{equation*}
            \calN = \set{ x \in \calS }{ \forall i \in [s+1] ,~ f(x,\bftheta_i) = 0 }.
        \end{equation*}
    \end{theorem}

    \begin{remark}[Lower dimensional semialgebraic sets have rare closures]
        Since $\calR \subset \bbR^{p \times (s+1)}$ has dimension (strictly) less than $p(s+1)$, the same is true for its closure in the Euclidean and Zariski topology \cite[Proposition~2.8.2 on p.~50]{Bochnak1998Real}. Therefore, $\calR$ is rare/nowhere dense in both \cite[Proposition~2.8.4 on p.~51]{Bochnak1998Real}.
    \end{remark}

    We now present our result on orbit separation for $\SortEmbeddingTwo{\bfA}{\bfB}$.

    \begin{theorem}\label{thm:dimensionreduction}
     Let $d,n,D$ be natural numbers.    If $\ntemplates \geq (2\nrows - 1)\ncolumns$, there exists a semialgebraic set $\calR \subset \bbR^{(\nrows+\ncolumns) \times \ntemplates} \simeq \bbR^{\ncolumns \times \ntemplates} \times \bbR^{\nrows \times \ntemplates}$ of dimension strictly less than $(\nrows + \ncolumns)\ntemplates$ such that $\smash{\QuotientSortEmbeddingTwo{\bfA}{\bfB}}$ is injective for all $(\bfA,\bfB) \not\in \mathcal{R}$. Equivalently, $\smash{\QuotientSortEmbeddingTwo{\bfA}{\bfB}}$ is injective for {\em generic} pairs  $(\bfA,\bfB)\in\bbR^{(\nrows+\ncolumns) \times \ntemplates} $, where {\em generic} is understood in the sense of the Zariski topology.
    \end{theorem}

    \begin{proof}
       First, we make the simple observation that it suffices to prove the claim for $\ntemplates = (2\nrows - 1)\ncolumns$ since adding more measurements to an already injective map can never result in a map that is not injective. 

       Now, the main observation the proof is based on is that the symmetries of $\beta_\bfA$ can be exploited to reduce the dimension of the domain on which injectivity needs to be proven. The first of these symmetries is homogeneity: for all $t>0$ we have that $\beta_\bfA(t\bfX)=t\beta_\bfA(\bfX) $. The second symmetry is translation: namely, when applying a translation of $\bfX$ by a vector $\bfz\in \RR^d $ we obtain 
       \begin{equation}\label{eq:translation_symmetry}
        \SortEmbedding{\bfA}(\bfX+\bfone_\nrows \bfz^\top)=\SortEmbedding{\bfA}(\bfX)+\SortEmbedding{\bfA}(\bfone_\nrows\bfz^\top).
       \end{equation}
      Due to these symmetries, it suffices to show that $\SortEmbeddingTwo{\bfA}{\bfB}(\bfX) = \SortEmbeddingTwo{\bfA}{\bfB}(\bfY)$ implies $\bfX \sim_{S_\nrows} \bfY$ on the semialgebraic set 
        \begin{equation}\label{eq:calS}
        	\calS := \{(\bfX,\bfY) \in \mathbb{R}^{\nrows \times \ncolumns} \times \mathbb{R}^{\nrows \times \ncolumns} \,|\, \bfone_\nrows^\top \bfX = \bfnull_\ncolumns,~\norm{\bfX}_\mathrm{F}^2+\norm{\bfY}_\mathrm{F}^2 = 1 \}
        \end{equation}
of dimension $(2\nrows - 1)\ncolumns - 1$: indeed, if the above is true and if $\SortEmbeddingTwo{\bfA}{\bfB}(\bfX) = \SortEmbeddingTwo{\bfA}{\bfB}(\bfY)$ for general $\bfX,\bfY \in \mathbb{R}^{\nrows \times \ncolumns}$, 
then we may subtract the column wise mean of $\bfX$ from both $\bfX$ and $\bfY$ and normalize\footnotemark ~the result to obtain a tuple $(\bfX',\bfY') \in \calS$, which due to homogeneity and the translation symmetry will satisfy  $\SortEmbeddingTwo{\bfA}{\bfB}(\bfX') = \SortEmbeddingTwo{\bfA}{\bfB}(\bfY')$. By assumption, we have $\bfX' \sim_{S_\nrows} \bfY'$ which, in turn, implies that $\bfX \sim_{S_\nrows} \bfY$.
\footnotetext{This normalization will not be possible if both $\bfX$ and $\bfY$ are zero after translation by the mean of $\bfX$ but in this case $\bfX=\bfY$.}

Now, consider the semialgebraic function $f : \calS \times \bbR^{\nrows + \ncolumns} \to \bbR$ given by 
        \begin{equation*}
            f((\bfX,\bfY),(\bfa,\bfb)) := \mathbf{b}^\top \left( \sort{\mathbf{X} \mathbf{a}} - \sort{\mathbf{Y} \mathbf{a}} \right),
        \end{equation*}
        for $(\bfX,\bfY) \in \calS$ and $(\bfa,\bfb) \in \bbR^\ncolumns \times \bbR^\nrows \simeq \bbR^{\nrows + \ncolumns}$. The set 
        \begin{equation*}
            \calN = \lbrace (\bfX,\bfY) \in \calS \,|\, \forall (\bfa,\bfb) \in \bbR^\ncolumns \times \bbR^\nrows ,~f((\bfX,\bfY),(\bfa,\bfb)) = 0 \rbrace
        \end{equation*}
        is exactly $\set{(\bfX,\bfY) \in \calS}{\bfX \sim_{S_\nrows} \bfY}$: indeed, fix arbitrary $\bfa \in \bbR^\ncolumns$ and note that the fact that  $    \forall \bfb \in \bbR^\nrows ,~ f((\bfX,\bfY),(\bfa,\bfb)) = 0  $ implies that $\sort{\mathbf{X} \mathbf{a}} - \sort{\mathbf{Y} \mathbf{a}}  $ is orthogonal to all vectors $\bfb \in \bbR^\nrows $, and therefore, that  $\sort{\mathbf{X} \mathbf{a}} -=\sort{\mathbf{Y} \mathbf{a}}  $.
        
        Since $\bfa \in \bbR^\ncolumns$ was arbitrary, the above continues to hold for the columns of a full spark matrix $\bfA \in \bbR^{\ncolumns \times \ntemplates'}$ with $\ntemplates' > \nrows(\ncolumns-1)$. Therefore, Theorem~\ref{thm:orbit_separation_beta} implies that $\bfX \sim_{S_\nrows} \bfY$. We have shown   that $\calN \subset \set{(\bfX,\bfY) \in \calS}{\bfX \sim_{S_\nrows} \bfY}$. The reverse direction is obvious.
        
        In the proof of \cite[Proposition~3.1 on p.~393]{Dym2024LowDimensional}, it is shown that
        \begin{equation*}
            \dim \set{ (\bfa,\bfb) \in \bbR^\ncolumns \times \bbR^\nrows }{ f((\bfX,\bfY),(\bfa,\bfb)) = 0 } < \nrows + \ncolumns
        \end{equation*}
        for all $(\bfX,\bfY) \in \calS \setminus \calN$. Therefore, the finite witness theorem implies that there exists a semialgebraic set $\calR \subset \bbR^{(\nrows + \ncolumns) \times (2 \nrows - 1) \ncolumns}$ of dimension (strictly) less than $(\nrows + \ncolumns)(2 \nrows - 1) \ncolumns$ such that for all $(\bfA,\bfB) := ((\bfa_1~\dots~\bfa_\ntemplates),(\bfb_1~\dots~\bfb_\ntemplates)) \not\in \mathcal{R}$, we have that
        \begin{align*}
            \MoveEqLeft[3] \set{(\bfX,\bfY) \in \calS}{\bfX \sim_{S_\nrows} \bfY} \\
            ={}& \lbrace (\bfX,\bfY) \in \calS \,|\, \forall i \in [(2 \nrows - 1) \ncolumns],~f((\bfX,\bfY),(\bfa_i,\bfb_i)) = 0 \rbrace \\
	        ={}& \set{ (\bfX,\bfY) \in \calS }{ \SortEmbeddingTwo{\bfA}{\bfB}(\mathbf{X}) = \SortEmbeddingTwo{\bfA}{\bfB}(\mathbf{Y}) }.
        \end{align*}
        The theorem is thus proven.
    \end{proof}

 We now present our result on orbit separation for $\beta_{\bfA,\linop}=\linop\circ\beta_\bfA$.
 
\begin{theorem}\label{thm:dimensionreduction2}
  Let $n,d, D, M$ be natural numbers so that  $D\geq n(d-1)+1$ and $M\geq (2\nrows - 1)\ncolumns$. Let $\bfA\in \RR^{d\times D}$ be a full spark matrix. Then there exists a closed algebraic set $\calR \subset \set{\linop:\bbR^{n\times D}\rightarrow\bbR^{M}}{\linop\,{\rm linear}}\simeq\bbR^{ M \times(\nrows D)} $ of dimension strictly less than $\nrows  DM$, such that $\bar \beta_{\bfA,\linop}$ is injective for all $\linop \not\in \mathcal{R}$.  Consequently, $\bar \beta_{\bfA,\linop}$ is injective for {\em generic} pairs  $(\bfA,\linopmat)\in\bbR^{\ncolumns \times \ntemplates} \times \bbR^{M\times (nD)}$, where {\em generic} is understood in the sense of the Zariski topology.
\end{theorem}
\begin{proof}

This proof uses elementary results from linear algebra and constructs a closed algebraic set $\calR$ that satisfies the desired properties.  As in the previous theorem, we may assume without loss of generality that $M=2nd-d$.

First, recall that, if $\bfA$ has full spark, then by Theorem~\ref{thm:orbit_separation_beta}, the map $\bar{\beta}_\bfA$ is injective. Next, let
\[
    \bfP = (\bfP_1,\ldots,\bfP_D,\bfP_{D+1},\ldots,\bfP_{2D})\in S_n^{2D}.
\]
Define the linear map $\Phi_\bfP:\bbR^{n\times d}\times\bbR^{n\times d}\rightarrow\bbR^{n\times D}$ by
\[
    \Phi_\bfP(\bfX,\bfY) := \begin{pmatrix}
        (\bfP_i \bfX - \bfP_{D+i}\bfY)\bfa_i
    \end{pmatrix}_{i=1}^{D}.
\]
Observe that, for any $\bfz \in\bbR^d$, we have 
\[
    \Phi_\bfP(\bfone_\nrows\bfz^\top,\bfone_\nrows\bfz^\top)=\bfnull_{n \times D}
\]
so $\dim \ker(\Phi_\bfP) \geq d$, and by the rank-nullity theorem, $\dim \Ran(\Phi_\bfP)\leq 2nd-d$. 

Next, observe that the set
\[ \set{ \beta_{\bfA}(\bfX)-\beta_{\bfA}(\bfY) }{(\bfX,\bfY)\in\bbR^{n\times d}\times\bbR^{n\times d} } \]
is contained in 
\[ {\cal W}:=\bigcup_{\bfP\in S_n^{2D}} \Ran(\Phi_\bfP).\]
The set ${\cal W}$ is a finite union of linear subspaces, each of dimension at most $2nd-d$,  and hence an algebraic set. 

For each $\bfP\in(S_n)^{2D}$, let $\set{ \bfe_i^{(\bfP)} }{ 1\leq i \leq \dim\Ran(\Phi_\bfP)}$ be a basis for $\Ran(\Phi_\bfP)$. Define
\[ {\calR} = \bigcup_{\bfP\in(S_n)^{2D}}
{\cal R}_\bfP\]
where 
\[ {\cal R}_\bfP :=
\set*{ \linopmat\in\bbR^{M\times nD} }{ \ker(\linopmat) \cap \Ran(\Phi_\bfP) \neq\{\bfnull_{nD}\} }. \]
We claim that that each ${\cal R}_{\bfP}$ is  a closed algebraic subset of dimension strictly less than $nDM$, and hence ${\cal R}$ itself is a closed algebraic set of dimension less than $nDM$.

To show this, fix $\bfP\in(S_n)^{2D}$ and let $p = \dim\Ran(\Phi_\bfP)$. Define a matrix $\boldsymbol{M}\in\bbR^{M\times p}$ whose $i^{\rm th}$ column is $\smash{\linopmat \mathbf e_i^{(\bfP)} \in\bbR^{M}}$, for $1\leq i\leq p$. Then, $\linopmat\in\calR_\bfP$ if and only if $\operatorname{rank}(\boldsymbol{M})<p$. Since $M \geq 2nd-d\geq p$, this condition is equivalent to the vanishing of all $p \times p$ minors of $\boldsymbol{M}$, which can be expressed as polynomial equations in the entries of $\linopmat$. Hence, $\calR$ is a closed algebraic set. 

To show that its dimension is strictly less than $nDM$ (the dimension of the ambient space of linear operators $\linop:\bbR^{n\times D}\rightarrow \bbR^{M}$), it suffices to show that the complement of $\calR_\bfP$ is nonempty. 
In other words, we need to show there exists some $\linopmat$ such that $\ker(\linopmat)\cap\Ran(\Phi_\bfP)=\{\bfnull_{nD}\}$. 
To construct such an $\linopmat$, consider a full-rank $\linopmat_1 \in\bbR^{M \times nD}$. Then, $\dim\ker(\linopmat_1)=nD-M\leq nD-p$. The orthogonal complement $\Ran(\Phi_\bfP)^\perp$ has dimension $nD-p$. 
Thus, we can choose an invertible (even orthogonal) transformation $\mathbf{T}$ such that $\mathbf{T} \ker(\linopmat_1)\subset \Ran(\Phi_\bfP)^{\perp}$. Define $\linopmat = \linopmat_1 \mathbf{T}^{-1}$. 
Then $\ker(\linopmat)=\mathbf{T} \ker(\linopmat_1)$, and so $\ker(\linopmat)\perp\Ran(\Phi_\bfP)$, implying $\ker(\linopmat)\cap\Ran(\Phi_\bfP)=\{\bfnull_{nD}\}$. Thus, $\calR_\bfP$ has nonempty complement and dimension stricly less than $nDM$. This proves the claim.

 Finally, suppose $\linopmat \not\in {\cal R}$. Then, for all $\bfP\in(S_n)^{2D}$, we have $\ker(\linopmat) \cap \Ran(\Phi_\bfP) = \{\bfnull_{nD}\}$, which implies $\ker(\linopmat)\cap {\cal W}=\{\bfnull_{nD}\}$. Now, suppose $\bfX,\bfY\in\bbR^{n\times d}$ satisfy  $\beta_{\bfA,\linopmat}(\bfX)=\beta_{\bfA,\linopmat}(\bfY)$. Then, $\beta_{\bfA}(\bfX)-\beta_{\bfA}(\bfY) \in {\cal W}\cap\ker(\linopmat) =\{\bfnull_{nD}\}$. Since $\bar{\beta}_\bfA$ is injective, it follows that $\bfY=\mathbf{Q} \bfX$ for some permutation matrix $\mathbf{Q}\in S_n$. This concludes the proof.
\end{proof}

    \begin{remark}[On a generalization due to two of the authors]
        Two of the authors of this paper generalized the above idea of dimension reduction using symmetries to the more general setting in which a finite group $G$ acts by isometries on a $d_V$-dimensional real vector space $V$ \cite[Theorem~1.6 on p.~5]{Balan2023GInvariantI}: if $d_G$ denotes the dimension of the subspace of invariants $\set{ v \in V }{ \forall g \in G,~ g v = v }$, then a fairly generic embedding into $\bbR^{2 d_V - d_G}$ achieves orbit separation.
    \end{remark}

    \section{Lipschitz Distortion Bounds}\label{ssec:BiLipschitzCondtion}
    In this section, we will bound the bi-Lipschitz distortion of $\beta_\bfA$. We recall that the the upper Lipschitz constant is given by the largest singular value $\sigma_1(\bfA)$. We do not have such a simple characterization for the lower bound. In this section, we will show how to estimate the lower bound via the notion of projective uniformity. We will then use projective uniformity to get estimates on the lower Lipschitz constant of $\bfA$ as a function of $(n,d)$, ultimately obtaining a bi-Lipschitz distortion proportional to $n^2$. We will also show that the bi-Lipschitz distortion cannot be better than $\sim n^{1/2}$, and show how to extend our positive results to $\beta_{\bfA,\linop}$. 

    \subsection{Upper Distortion Bounds Based on Projective Uniformity}\label{sssec:UpperDistortionBounds}

    Let us first introduce projective uniformity as defined in \cite{Cahill2024GroupInvariant}. We are interested in matrices $\bfA \in \bbR^{\ncolumns \times \ntemplates}$ which satisfy conditions of the form 
    \begin{equation}\label{eq:projectiveuniformity}
        \sort{\abs{\bfA^\top \bfe}}_{D-m+1} \geq \delta, \qquad \mbox{for all } \bfe \in S^{d-1},
    \end{equation}
    where $m \in [\ntemplates]$ and $\delta > 0$; i.e., the $m$-th smallest entry of the vector $(\abs{\bfa_k^\top \bfe})_{k = 1}^\ntemplates$ exceeds $\delta$: the authors of \cite{Cahill2024GroupInvariant} call this property of the columns of $\bfA$ \emph{$(m,\delta)$-projective uniformity}. 

    When the above inequality is satisfied, we may derive a simple lower bound on the lower Lipschitz constant of $\QuotientSortEmbedding{\bfA}$.

    \begin{theorem}\label{thm:biLipschitzblueprint}
        Let $d,n,D$ be natural numbers, and let  $\bfA \in \RR^{d \times D}$ satisfy equation~\eqref{eq:projectiveuniformity} with $\delta > 0$ and $m \in [\ntemplates]$ such that $\nrows^2(m-1) \leq \ntemplates$. Then, the lower Lipschitz constant of $\QuotientSortEmbedding{\bfA}$ is greater than or equal to $\delta \sqrt{D - \nrows^2(m-1)}$.
    \end{theorem}
	
 \begin{proof}
	Let $\bfX, \bfY \in \bbR^{\nrows \times \ncolumns}$ be arbitrary but fixed with rows $(\bfx_i)_{i = 1}^\nrows, (\bfy_i)_{i = 1}^\nrows$, respectively, and let $(\bfa_k)_{k = 1}^\ntemplates$ denote the columns of $\bfA \in \bbR^{\ncolumns \times \ntemplates}$. Due to \eqref{eq:projectiveuniformity}, for each fixed $i,j$, there will be at most $m-1$ indices $k\in [D] $ for which 
	\begin{equation}\label{eq:kij}\bfa_k^T(\bfx_i-\bfy_j) \geq \delta \|\bfx_i-\bfy_j\|_2\end{equation}
	does not hold. It follows that there will be less than or equal to $n^2(m-1)$ indices $k$ for which this inequality does not hold for some $i,j$. Let $J\subset [D]$ be the set of indices for which \eqref{eq:kij} \emph{does} hold for all $i,j$ simultaneously. Then the cardinality of this set is greater than or equal to $D-n^2(m-1)$, and for appropriate permutations $\sigma_1,\ldots,\sigma_D\in S_n$, we have that
	\begin{align*}
		\norm{\SortEmbedding{\bfA}(\bfX) - \SortEmbedding{\bfA}(\bfY)}_\mathrm{F}^2 &= \sum_{k = 1}^\ntemplates \sum_{i= 1}^\nrows \abs{(\bfx_i - \bfy_{\sigma_k(j)})^\top \bfa_k}^2 \geq \sum_{k\in J} \sum_{i= 1}^\nrows \abs{(\bfx_i - \bfy_{\sigma_k(j)})^\top \bfa_k}^2 \\
		&\geq \sum_{k\in J} \delta^2 \sum_{i= 1}^\nrows \|\bfx_i - \bfy_{\sigma_k(j)}\|^2 \geq \delta^2 \abs{J} \cdot \distance{\bfX}{\bfY}^2 \\
	  &\geq \delta^2 \left(\ntemplates - \nrows^2 (m - 1)\right) \cdot \distance{\bfX}{\bfY}^2.
	\end{align*}
	Taking the root of this inequality yields the advertised result.
\end{proof}

\subsection{Constructing Projectively Uniform Matrices}
We will now give three different constructions of projective uniform matrices $\bfA$, which will lead to quantitative bounds on the distortion of $\SortEmbedding{\bfA}$. The first construction will be deterministic but only for the case $d=2$. In this case we will get a distortion proportional to $n^2$ while using a similar dimension $D=n^2$. The next two constructions will be probabilistic. We will show that for $D$ large enough ($D \gtrsim n^{4}d$ in the second construction and $D \gtrsim n^{2}d$ up to logarithmic factors in the third construction), with high probability, we will get $\bfA$ with a distortion proportional to $n^2$ (in the third construction this will be up to logarithmic factors)

 \paragraph{First Construction: A Non-Probabilistic Construction with Distortion in $O(\nrows^2)$}\label{par:FirstConstruction} We begin with a simple non-probabilistic construction for the case $\ncolumns=2$, which achieves distortion of at most $2\nrows^2$ using $\ntemplates = 4\nrows^2 $ vectors: consider the matrix $\bfA \in \bbR^{2 \times \ntemplates}$ with columns
\begin{equation*}
	\bfa_k := \begin{pmatrix}
		\cos(2\pi k/\ntemplates) \\
		\sin(2\pi k/\ntemplates)
	\end{pmatrix}, \qquad k \in [\ntemplates].
\end{equation*}

Then, $\bfA$ satisfies equation~\eqref{eq:projectiveuniformity} with $m=3$ and an appropriate $\delta > 0$: indeed, let
\begin{equation*}
	\bfx = \begin{pmatrix}
		\cos(\theta) \\
		\sin(\theta)
	\end{pmatrix} \in S^{1}
\end{equation*}
be arbitrary where $\theta \in [0,2\pi)$ and denote $\theta_{\pm} := \theta \pm \pi/2 \mod 2 \pi$. Since the columns $\bfa_k$ are equidistributed on the unit sphere, there is at most one $k \in [\ntemplates]$ such that $\abs{2 \pi k / \ntemplates - \theta_-} < \pi/\ntemplates$ and at most one $k \in [\ntemplates]$ such that $\abs{2 \pi k / \ntemplates - \theta_+} < \pi/\ntemplates$. Excluding these columns from consideration and assuming that $2 \pi k / \ntemplates$ is closer to $\theta_-$ than $\theta_+$, we may estimate 
\begin{equation*}
	\abs{\bfa_k^\top \bfx} = \abs*{\cos\left(\frac{2\pi k}{\ntemplates} - \theta\right)} = \abs*{\sin\left(\frac{2\pi k}{\ntemplates} - \theta_-\right)}.
\end{equation*}
Notably, $\pi/\ntemplates \leq \abs{2 \pi k / \ntemplates - \theta_-} \leq \pi / 2$ such that the simple inequality $\abs{\sin(x)} \geq 2 \abs{x} / \pi$ for $x \in [-\pi/2,\pi/2]$ shows that 
\begin{equation*}
	\abs{\bfa_k^\top \bfx} \geq \frac{2}{\pi} \abs*{\frac{2\pi k}{\ntemplates} - \theta_-} \geq \frac{2}{\ntemplates} =: \delta.
\end{equation*}
The case in which $2 \pi k / \ntemplates$ is closer to $\theta_+$ than $\theta_-$ is dealt with analogously.

According to Theorem~\ref{thm:biLipschitzblueprint}, it follows that the lower Lipschitz constant of $\QuotientSortEmbedding{\bfA}$ is lower bounded by 
\begin{equation*}
	\frac{2}{\ntemplates} \sqrt{\ntemplates - 2 \nrows^2} = \frac{1}{\sqrt{2} \nrows}.
\end{equation*}
At the same time, the upper Lipschitz constant is the largest singular value of $\bfA$ which is just $\sqrt{\ntemplates/2} = \sqrt{2} \nrows$ since 
\begin{equation*} \bfA\bfA^T =  \sum_{k=1}^D \mathbf{a}_k \mathbf{a}_k^\top =\sum_{k=1}^D \begin{pmatrix}
 \cos^2\!\left(\tfrac{2\pi k}{D}\right) &  \cos\!\left(\tfrac{2\pi k}{D}\right)\sin\!\left(\tfrac{2\pi k}{D}\right) \\[0.5em]
 \cos\!\left(\tfrac{2\pi k}{D}\right)\sin\!\left(\tfrac{2\pi k}{D}\right) &  \sin^2\!\left(\tfrac{2\pi k}{D}\right)
\end{pmatrix} = \frac{D}{2} \mathbf{I}_2,
\end{equation*}
 which in turn follows from the identities
\begin{align*} &\sum_{k=1}^D \cos^2\!\left(\tfrac{2\pi k}{D}\right) 
= \sum_{k=1}^D \sin^2\!\left(\tfrac{2\pi k}{D}\right) 
= \frac{D}{2} , \\
 &~\sum_{k=1}^D \cos\!\left(\tfrac{2\pi k}{D}\right)\sin\!\left(\tfrac{2\pi k}{D}\right) = 0.\end{align*}
 Therefore, the distortion in this setup is at most $2 \nrows^2$.

    \paragraph{Second Construction: Gaussian Matrices}\label{par:SecondConstruction} Random matrices $\bfA \in \bbR^{\ncolumns \times \ntemplates}$ may satisfy equation~\eqref{eq:projectiveuniformity} with high probability. Potentially, the simplest examples are Gaussian random matrices as shown in the following result, which combines an idea from the proof of \cite[Lemma~23]{Cahill2024GroupInvariant} with the general strategy outlined in \cite{Abdalla2025Recovery}.

    \begin{proposition}\label{prop:projectiveuniformityGaussian}
        Let $\bfA \in \bbR^{\ncolumns \times \ntemplates}$ be a matrix with independent standard normal entries and let $\lambda \in [\ntemplates]/\ntemplates$. Then, 
        \begin{equation*}
            \mathbb{P}\left\lbrace \forall \bfx \in S^{d-1} ,~ \sort{\abs{\bfA^\top \bfx}}_{D - \lambda \ntemplates+1} \geq \frac{\sqrt{\pi}}{3\sqrt{2}}\lambda \right\rbrace \geq 1 - \exp\left(-\frac29 \lambda^2 D\right)
        \end{equation*}
        if $D \gtrsim d/\lambda^2$.
    \end{proposition}

    \begin{proof}
        Inspired by \cite[Lemma~23]{Cahill2024GroupInvariant}, we will show that
        \begin{equation*}
            \min_{\bfx \in S^{d-1}} \sum_{k = 1}^{D} K_{\lbrace \abs{\bfa_k^\top \bfx} \geq \delta \rbrace} > (1-\lambda) D
        \end{equation*}
        with high probability, where $(\bfa_k)_{k = 1}^\ntemplates$ denote the columns of $\bfA$ and $\delta > 0$ is chosen appropriately. Add and subtract the mean,
        \begin{multline*}
            \min_{\bfx \in S^{d-1}} \frac{1}{D} \sum_{k = 1}^{D} K_{\lbrace \abs{\bfa_k^\top \bfx} \geq \delta \rbrace} \\
            = \min_{\bfx \in S^{d-1}} \left( \mathbb{P}\left\lbrace \abs*{\bfa^\top \bfx} \geq \delta \right\rbrace - \mathbb{P}\left\lbrace \abs*{\bfa^\top \bfx} \geq \delta \right\rbrace +  \frac{1}{D} \sum_{k = 1}^{D} K_{\lbrace \abs{\bfa_k^\top \bfx} \geq \delta \rbrace} \right),
        \end{multline*}
        and note that, due to the rotation symmetry of the multivariate standard normal distribution, it holds that 
        \begin{align*}
            \mathbb{P}\left\lbrace \abs*{\bfa^\top \bfx} \geq \delta \right\rbrace &= \mathbb{P}\left\lbrace \abs*{a_1} \geq \delta \right\rbrace = 1 - \mathbb{P}\left\lbrace \abs*{a_1} < \delta \right\rbrace = 1-\frac{1}{\sqrt{2 \pi}} \int_{-\delta}^{\delta} \rme^{-t^2/2} \dd t \\
            &\geq 1-\sqrt{\frac{2}{\pi}} \delta,
        \end{align*}
        for $\delta \in [0,1]$. Plugging this back in yields
        \begin{multline*}
            \min_{\bfx \in S^{d-1}} \frac{1}{D} \sum_{k = 1}^{D} K_{\lbrace \abs{\bfa_k^\top \bfx} \geq \delta \rbrace} \\
            \geq 1 - \sqrt{\frac{2}{\pi}} \delta - \max_{\bfx \in S^{d-1}} \left( \mathbb{P}\left\lbrace \abs*{\bfa^\top \bfx} \geq \delta \right\rbrace - \frac{1}{D} \sum_{k = 1}^{D} K_{\lbrace \abs{\bfa_k^\top \bfx} \geq \delta \rbrace} \right).
        \end{multline*}
        By the bounded difference inequality \cite[e.g.~Theorem~5.7.1 on p.~165]{Vershynin2025HighDimensional}, we have that 
        {\allowdisplaybreaks
        \begin{align*}
            \MoveEqLeft[3] \min_{\bfx \in S^{d-1}} \frac{1}{D} \sum_{k = 1}^{D} K_{\lbrace \abs{\bfa_k^\top \bfx} < \delta \rbrace} \\
            >{}& 1 - \sqrt{\frac{2}{\pi}} \delta - \mathbb{E} \max_{\bfx \in S^{d-1}} \Bigg( \mathbb{P}\left\lbrace \abs*{\bfa^\top \bfx} \geq \delta \right\rbrace - \frac{1}{D} \sum_{k = 1}^{D} K_{\lbrace \abs{\bfa_k^\top \bfx} \geq \delta \rbrace} \Bigg) - t \\
            \geq{}& 1 - \sqrt{\frac{2}{\pi}} \delta - \mathbb{E} \max_{\bfx \in S^{d-1}} \Bigg\lvert \frac{1}{D} \sum_{k = 1}^{D} K_{\lbrace \abs{\bfa_k^\top \bfx} \geq \delta \rbrace} - \mathbb{P}\left\lbrace \abs*{\bfa^\top \bfx} \geq \delta \right\rbrace \Bigg\rvert - t
        \end{align*}
        with probability greater than or equal to $1-\exp(-2 t^2 D)$. Finally, the VC law of large numbers \cite[e.g.~Theorem~8.3.15 on p.~237]{Vershynin2025HighDimensional} implies that}
        \begin{align*}
            \min_{\bfx \in S^{d-1}} \frac{1}{D} \sum_{k = 1}^{D} K_{\lbrace \abs{\bfa_k^\top \bfx} < \delta \rbrace} > 1 - \sqrt{\frac{2}{\pi}} \delta - C \sqrt{\frac{d}{D}} - t,
        \end{align*}
        where $C > 0$ is an absolute constant. Here, we use that 
        \begin{equation*}
            K_{\lbrace \abs{\bfa^\top \bfx} \geq \delta \rbrace} = \max\lbrace K_{\lbrace \bfa^\top \bfx \geq \delta \rbrace}, K_{\lbrace \bfa^\top \bfx \leq -\delta \rbrace} \rbrace
        \end{equation*}
        and that the function classes $\set{\bfa \mapsto K_{\lbrace (\pm \bfa)^\top \bfx \geq \delta \rbrace}}{\bfx \in S^{d-1}}$ of indicators of half-spaces have VC dimension $d$ such that \cite[Proposition~8.3.11 on p.~234]{Vershynin2025HighDimensional} shows that the VC dimension of $\set{\bfa \mapsto K_{\lbrace \abs{\bfa^\top \bfx} \geq \delta \rbrace}}{\bfx \in S^{d-1}}$ is less or equal than $10 d$. Finally, it remains to balance the parameters: the simple choices 
        \begin{equation*}
            \delta := \frac{\sqrt{\pi}}{3\sqrt{2}}\lambda, \qquad D \geq 9C^2 \frac{d}{\lambda^2}, \qquad t = \frac{\lambda}{3}
        \end{equation*}
        finish the proof.
    \end{proof}

    Combining the two prior results yields the following bound on the lower Lipschitz constant of $\QuotientSortEmbedding{\bfA}$ when $\bfA \in \bbR^{\ncolumns \times \ntemplates}$ is Gaussian; it follows immediately that the distortion of $\QuotientSortEmbedding{\bfA}$ is in $O(\nrows^2)$, which notably is independent of the number of columns $\ncolumns$ of $\bfA$.

    \begin{theorem}\label{thm:biLipschitzprobabilistic}
      Let $d,n,D$ be natural numbers.   Let $\bfA \in \RR^{d \times D}$ be a matrix with independent standard normal entries. Then,
        \begin{multline}\label{eq:probboundlowLipschitzconst}
            \mathbb{P}\left\lbrace \forall \bfX,\bfY \in \bbR^{n \times d},~ \norm{\beta_\bfA(\bfX) - \beta_\bfA(\bfX)}_2 \geq \frac{\sqrt{2 \pi}}{9\sqrt{3}} \frac{\sqrt{D}}{n^2} \cdot \operatorname{dist}(\bfX,\bfY) \right\rbrace \\
            \geq 1 - \exp\left(-\frac{8}{81} \frac{D}{n^4}\right)
        \end{multline}
        and the distortion of $\QuotientSortEmbedding{\bfA}$ is in $O(\nrows^2)$ with probability greater than or equal to $1 - 2 \exp(-c_1\ntemplates) - \exp(-c_2 \nrows^{-4} \ntemplates )$, where $c_1,c_2 > 0$ are universal constants, provided that $\ntemplates \gtrsim \nrows^4 \ncolumns$.
    \end{theorem}

    \begin{proof}
        Consider an arbitrary $\lambda \in [\ntemplates]/D$ with $\lambda \leq \nrows^{-2} + \ntemplates^{-1}$ and suppose that we are in the highly likely event whose probability is estimated in Proposition~\ref{prop:projectiveuniformityGaussian}. Then, Theorem~\ref{thm:biLipschitzblueprint} shows that the lower Lipschitz constant of $\QuotientSortEmbedding{\bfA}$ is greater than or equal to 
        \begin{equation*}
            \frac{\sqrt{\pi}}{3\sqrt{2}} \sqrt{\ntemplates} \cdot \lambda \sqrt{1 - \nrows^2\left(\lambda - \frac{1}{\ntemplates}\right)}.
        \end{equation*}

        We note that $\lambda \mapsto \lambda^2 (1 - \nrows^2 \lambda)$ attains its maximum at $\lambda_\ast = 2/3\nrows^2$. It therefore seems to be a good idea to set $\lambda = \lceil 2D / (3n^2) \rceil/D$ so that $\lambda \geq 2/3n^2 \geq \lambda - 1/D$, and so
        \begin{equation*}
            \frac{\sqrt{\pi}}{3\sqrt{2}} \sqrt{\ntemplates} \cdot \lambda \sqrt{1 - \nrows^2\left(\lambda - \frac{1}{\ntemplates}\right)} \geq \frac{\sqrt{2 \pi}}{9\sqrt{3}} \frac{\sqrt{\ntemplates}}{\nrows^2}.
        \end{equation*}
        Equation~\eqref{eq:probboundlowLipschitzconst} with $D \gtrsim n^{4} d$ follows after plugging in our choice for $\lambda$ in the statement of Proposition~\ref{prop:projectiveuniformityGaussian} and applying Theorem~\ref{thm:biLipschitzblueprint}.

        For the claim about the distortion of $\QuotientSortEmbedding{\bfA}$, note that the upper Lipschitz constant of $\QuotientSortEmbedding{\bfA}$ is the largest singular value $\sigma_1(\bfA)$ (cf.~Theorem~\ref{thm:BHS22}). When $\bfA \in \bbR^{\ncolumns \times \ntemplates}$ is Gaussian, then its largest singular value is (strictly) less than $\sqrt{\ntemplates} + \sqrt{\ncolumns} + t$ with probability greater than or equal to $1 - 2\exp(-c_1t^2)$, where $c_1 > 0$ is a universal constant \cite[Corollary~7.3.2 on p.~204]{Vershynin2025HighDimensional}. If we pick $t = \sqrt{\ntemplates}$, then a union bound shows that the distortion of $\QuotientSortEmbedding{\bfA}$ is in $O(\nrows^2)$ with probability greater than or equal to $1 - 2 \exp(-c_1\ntemplates) - \exp(-c_2 \nrows^{-4} \ntemplates )$ when $\ntemplates \gtrsim \nrows^4 \ncolumns$, where $c_1 = 8/81$.
    \end{proof}

    \paragraph{Third Construction: Matrices whose Columns are Uniformly Sampled from the Unit Sphere}\label{par:ThirdConstruction} \cite[Lemma~23]{Cahill2024GroupInvariant} shows that random matrices $\bfA \in \bbR^{\ncolumns \times \ntemplates}$ whose columns are independently drawn from the uniform distribution on the unit sphere $S^{d-1}$ also satisfy equation~\eqref{eq:projectiveuniformity} with high probability. Combining this with Theorem~\ref{thm:biLipschitzblueprint} in a carbon copy of the proof above yields the following result.

    \begin{theorem}\label{thm:DistortionUnitSphere}
      Let $d,n,D$ be natural numbers.   Let $\bfA \in \RR^{\ncolumns \times \ntemplates}$ be a matrix whose columns are drawn independently from the uniform distribution on the unit sphere. Then, with probability greater than or equal to $1 - \exp(-\ntemplates / 18\nrows^2)$, the lower Lipschitz constant of $\QuotientSortEmbedding{\bfA}$ is greater than or equal to 
        \[
            \frac{\sqrt{\pi}}{24 \sqrt{3}} \left(\ncolumns + 3 \log(\sqrt{6} \nrows)\right)^{-1/2} \frac{\sqrt{\ntemplates}}{\nrows^2},
        \]
        provided that
        \begin{equation}\label{eq:somewhat_complicated_condition_on_D}
            \ntemplates \geq 18 \ncolumns \nrows^2 \log\left( \frac{48 \sqrt{3} \nrows \sqrt{\ncolumns + 3 \log(\sqrt{6} \nrows)}}{\sqrt{\pi}} + 1 \right).
        \end{equation}
        Thus, with probability greater than or equal to $1 - 2 \exp(-D) - \exp(-\ntemplates / 18\nrows^2)$, the distortion of $\QuotientSortEmbedding{\bfA}$ is in $\widetilde O(\nrows^2)$.
    \end{theorem}

    \begin{proof}
        The lower bound on the lower Lipschitz constant of $\QuotientSortEmbedding{\bfA}$ follows from \cite[Lemma~23]{Cahill2024GroupInvariant} and Theorem~\ref{thm:biLipschitzblueprint}.
        
        For the estimate on the distortion of $\QuotientSortEmbedding{\bfA}$, note that the uniform distribution on the sphere $S^{\ncolumns-1}$ is subgaussian with subgaussian norm in $O(\ncolumns^{-1/2})$ \cite[Theorem~3.4.5 on p.~73]{Vershynin2025HighDimensional}. Therefore, the uniform distribution on the sphere $\sqrt{\ncolumns} S^{\ncolumns-1}$ is subgaussian with subgaussian norm in $O(1)$. Additionally, the uniform distribution on the sphere $\sqrt{\ncolumns} S^{\ncolumns-1}$ is isotropic \cite[Proposition~3.3.8 on p.~67]{Vershynin2025HighDimensional}. It follows from \cite[Theorem~4.6.1 on pp.~122--123]{Vershynin2025HighDimensional} that the largest singular value of $\bfA \in \bbR^{\ncolumns \times \ntemplates}$ satisfies
        \begin{equation*}
            \sigma_1(\bfA) = \frac{1}{\sqrt{\ncolumns}} \sigma_1(\sqrt{\ncolumns} \bfA^\top) \leq \sqrt{\frac{\ntemplates}{\ncolumns}} + C \left(1 + \frac{t}{\sqrt{\ncolumns}}\right)
        \end{equation*}
        with probability greater than or equal to $1 - 2 \exp(-t^2)$. Letting $t = \sqrt{\ntemplates}$ yields that 
        \begin{equation*}
            \mathbb{P}\left\lbrace \sigma_1(\bfA) \lesssim \sqrt{\frac{\ntemplates}{\ncolumns}} \right\rbrace \geq 1 - 2 \exp(-D),
        \end{equation*}
        which together with the bound on the lower Lipschitz constant (and a union bound) shows that the distortion of $\QuotientSortEmbedding{\bfA}$ is in $\widetilde O(\nrows^2)$, with probability greater than or equal to $1 - 2 \exp(-D) - \exp(-\ntemplates / 18\nrows^2)$.
    \end{proof}

    We note that the dependency on $n$ in the bound on the lower Lipschitz constant is worse by a logarithmic factor when compared to Theorem~\ref{thm:biLipschitzprobabilistic} but that the dependency on $n$ in $\ntemplates$ as well as in the bound on the probability is quadratic (up to logarithmic factors) instead of quartic.

      We now return to the Wasserstein interpretation discussed in
Section~\ref{subsec:related_sliced_wasserstein}. As observed there,
when the columns of $\bfA$ are sampled independently from the uniform
distribution on the unit sphere, the embedding
$\QuotientSortEmbedding{\bfA}$ corresponds (up to normalization) to the
Monte--Carlo approximation of the sliced $2$-Wasserstein distance based
on $D=\ntemplates$ random directions. Consequently, the distortion bounds
established in Theorem~\ref{thm:DistortionUnitSphere} translate directly
into quantitative comparisons between the sampled sliced Wasserstein
distance and the full $2$-Wasserstein distance on empirical measures.

    Therefore, Theorem~\ref{thm:DistortionUnitSphere} immediately implies the following corollary.

    \begin{corollary}\label{cor:WassersteinPositive}
        Let $d,n,D$ be natural numbers, with $\ntemplates \gtrsim \ncolumns \nrows^2 \log(\nrows \sqrt{\ncolumns + \log(\nrows)})$ (as in equation~\eqref{eq:somewhat_complicated_condition_on_D}) and let $(\bftheta_k)_{k=1}^D \in \bbR^d$ be drawn independently from the uniform distribution on the unit sphere. Then, with probability greater than or equal to $1 - 3 \exp(-\ntemplates / 18\nrows^2)$,
        \begin{equation*}
            \frac{1}{n^2 \sqrt{d + \log(n)}} \cdot \mathrm{W}_2(\mu,\nu) \lesssim \widetilde{\mathrm{SW}}_2(\mu,\nu;(\bftheta_k)_{k=1}^D) \lesssim \frac{1}{\sqrt{d}} \cdot \mathrm{W}_2(\mu,\nu)
        \end{equation*}
        for all uniform empirical measures $\mu$, $\nu$ over $n$ vectors in $\bbR^d$.
    \end{corollary}

This result naturally raises the question of whether similar bounds can be expected
when $\mu$ and $\nu$ are general probability measures on $\bbR^d$, rather than uniform
empirical measures. In particular, one may ask whether the sampled sliced Wasserstein
distance can provide a lower bound on the full Wasserstein distance that is uniform
in the number of samples.

    \begin{remark}[Foreshadowing Theorem~\ref{thm:LowerBoundDistortion}]
        In Theorem~\ref{thm:LowerBoundDistortion} (cf.~equation~\eqref{eq:I_cannot_think_of_anything}), we will show that there exist uniform empirical measures $\mu$ and $\nu$ over $n$ vectors in $\bbR^d$ such that, for all $(\bftheta_k)_{k = 1}^D \in S^{d-1}$, it holds that
        \begin{equation*}
            \widetilde{\mathrm{SW}}_2(\mu,\nu;(\bftheta_k)_{k=1}^D) \lesssim \sqrt{\frac{\sigma_1^2 + \sigma_2^2}{nD}} \cdot \mathrm{W}_2(\mu,\nu) \leq \frac{1}{\sqrt{n}} \cdot \mathrm{W}_2(\mu,\nu),
        \end{equation*}
        where $\sigma_1,\sigma_2 \geq 0$ are the two largest singular values of the matrix $\boldsymbol{\Theta}\in\bbR^{d \times D}$ whose columns are given by $(\bftheta_k)_{k=1}^D$. 

        This shows that no lower bound of the form $\widetilde{\mathrm{SW}}_2(\mu,\nu;(\bftheta_k)_{k=1}^D) \geq C \mathrm{W}_2(\mu,\nu)$ can hold with a constant $C > 0$ independent of $n$. Consequently, the sampled sliced Wasserstein distance cannot be bi-Lipschitz equivalent to the full Wasserstein distance with constants uniform in the support size. This observation is consistent with existing results showing that Wasserstein and sliced Wasserstein distances fail to be bi-Lipschitz equivalent in general \cite{Bayraktar2021Strong}.
    \end{remark}

    \subsection{A Universal Lower Bound on the Distortion}\label{ssec:LowerBoundDistortion}

    In all the constructions considered in the prior subsection, we had seen that the distortion grows in the number of rows of the matrices $\bfX \in \bbR^{\nrows \times \ncolumns}$. We will now show that one cannot hope to get rid of this growth in $\nrows$ completely: specifically, the distortion is at least in $\Omega(\nrows^{1/2})$.

\begin{theorem}\label{thm:LowerBoundDistortion}
Let $d,n,D$ be natural numbers and assume that $\ncolumns > 1$. Let
$\sigma_1(\bfA) \geq \sigma_2(\bfA) \geq \cdots \geq \sigma_{\ncolumns}(\bfA) \geq 0$
denote the singular values of the matrix $\bfA$. Then the lower Lipschitz constant of
$\QuotientSortEmbedding{\bfA}$ is less or equal than
\begin{equation*}
    \frac{(2 + 1/\nrows)^{1/2} \pi}{\nrows^{1/2}} \cdot
    \left( \sigma_{\ncolumns-1}^2(\bfA) + \sigma_{\ncolumns}^2(\bfA) \right)^{1/2} \lesssim
    \nrows^{-1/2} \cdot
    \left( \sigma_{\ncolumns-1}^2(\bfA) + \sigma_{\ncolumns}^2(\bfA) \right)^{1/2}.
\end{equation*}
Therefore, the distortion of $\QuotientSortEmbedding{\bfA}$ is in
$\Omega(\nrows^{1/2})$.
\end{theorem}

    \begin{proof}
        Let us consider the singular value decomposition $\bfA = \bfU \bfSigma \bfV$, with $\bfU \in \bbR^{\ncolumns \times \ncolumns}$, $\bfV \in \bbR^{\ntemplates \times \ntemplates}$ orthogonal matrices and $\bfSigma \in \bbR^{\ncolumns \times \ntemplates}$ containing the singular values of $\bfA$ on its diagonal. Then, we may assume, without loss of generality\footnote{Because $\beta_{\mathbf U\bfA}(\mathbf X) = \beta_{\bfA}(\mathbf X \mathbf U)$ and the map $\mathbf X \mapsto \mathbf X \mathbf U$ is an $S_n$-equivariant isometry, the maps $\beta_{\mathbf U\bfA}$ and $\beta_{\bfA}$ share the same lower and upper Lipschitz bounds.}, that
        \begin{equation*}
            \bfA = \left(\begin{array}{ccc|c}
                \sigma_1 & & & \\
                & \ddots & & \bfnull_{\ncolumns \times (\ntemplates - \ncolumns)} \\
                & & \sigma_d & 
            \end{array}\right) \begin{pmatrix}
                \mbox{---} & \overline \bfv_1 & \mbox{---} \\
                & \vdots & \\
                \mbox{---} & \overline \bfv_\ntemplates & \mbox{---} \\
            \end{pmatrix} = \begin{pmatrix}
                \mbox{---} & \sigma_1 \overline \bfv_1 & \mbox{---} \\
                & \vdots & \\
                \mbox{---} & \sigma_\ncolumns \overline \bfv_\ncolumns & \mbox{---} \\
            \end{pmatrix},
        \end{equation*}
        where $(\overline \bfv_i)_{i = 1}^\ntemplates \in \bbR^\ntemplates$ denote the row vectors of $\bfV$, which form an orthonormal basis of $\bbR^{\ntemplates}$. We denote the columns of $\bfA$ by $\bfa_k$.

        Now, consider the matrices $\bfX, \bfY \in \bbR^{\nrows \times \ncolumns}$ with rows 
        \begin{equation*}
            \bfx_i := \begin{pmatrix}
                \bfnull_{1 \times (\ncolumns-2)} & \cos( 2 \pi i / \nrows ) & \sin( 2 \pi i / \nrows )
            \end{pmatrix}
        \end{equation*}
        as well as $\bfy_1 := \bfnull_{1 \times \ncolumns}$ and $\bfy_i := \bfx_i$ for $i = 2,\dots,\nrows$. Since all rows of $\mathbf X$ and $\mathbf Y$ are supported on the last two coordinates, only the components of $\mathbf A$ along the singular directions corresponding to $\sigma_{d-1}$ and $\sigma_d$ contribute. Denote by $\mathbf a_k' \in \mathbb R^2$ the projection of $\mathbf a_k$ onto this two-dimensional subspace. Then, 
        \begin{equation*}
            \sum_{k = 1}^D \lVert \mathbf a_k' \rVert^2 = \sigma_{d-1}^2 + \sigma_d^2.
        \end{equation*}
        Moreover, direct computations show that $\distance{\bfX}{\bfY} = 1$ as well as 
        \begin{equation*}
            \norm{\SortEmbedding{\bfA}(\bfX) - \SortEmbedding{\bfA}(\bfY)}_\mathrm{F}^2 = \sum_{k = 1}^\ntemplates \norm{\sort{\bfX \bfa_k} - \sort{\bfY \bfa_k}}_2^2 \leq \sum_{k = 1}^\ntemplates \norm{(\bfX - \sigma_k \bfY) \bfa_k}_2^2,
        \end{equation*}
        for any choice of permutations $\sigma_k \in S_\nrows$.
        
        Let us choose the permutations in the following way: fix an arbitrary $k \in [\ntemplates]$ and let $i_k \in [\nrows]$ be such that $\bfx_{i_k}$ is almost orthogonal to $\bfa_k$; i.e., such that 
        \begin{equation*}
            \abs{\bfx_{i_k}^\top \bfa_k} = \abs*{ \begin{pmatrix}
                \cos( 2 \pi i_k / \nrows ) & \sin( 2 \pi i_k / \nrows )
            \end{pmatrix} \bfa_k' } \leq \frac{\pi}{\nrows} \norm{\bfa_k'}_2
        \end{equation*}
        where we used that the vectors $\bigl(\cos(2\pi i / \nrows), \sin(2\pi i / \nrows)\bigr)$ are equidistributed on the unit circle with angular spacing $2\pi / \nrows$. Consequently, there always exists such a vector whose angle with a unit vector orthogonal to $\bfa_k'$ is at most $\pi / \nrows$, which yields the stated bound. We will then define $\sigma_k \in S_\nrows$ by 
        \begin{equation*}
            \sigma_k(i) := \begin{cases}
                i+1 & \mbox{if } i < i_k, \\
                1 & \mbox{if } i = i_k, \\
                i & \mbox{if } i > i_k
            \end{cases}
        \end{equation*}
        provided that $i_k \leq \nrows / 2 + 1$ and otherwise
        \begin{equation*}
            \sigma_k(i) := \begin{cases}
                \nrows & \mbox{if } i = 1, \\
                i & \mbox{if } 1 < i < i_k, \\
                1 & \mbox{if } i = i_k, \\
                i-1 & \mbox{if } i > i_k.
            \end{cases}
        \end{equation*}
        (In this way, there are at most $\lceil \nrows / 2 \rceil$ mismatches on the unit circle.)

        Let us consider the case $i_k \leq \nrows / 2 + 1$ first. Let $a$ be the lower Lipschitz bound. We can estimate
        \begin{align*}
            a^2 &= a^2 \distance{\bfX}{\bfY}^2 \leq \norm{\SortEmbedding{\bfA}(\bfX) - \SortEmbedding{\bfA}(\bfY)}_\mathrm{F}^2 \leq \sum_{k = 1}^\ntemplates \norm{(\bfX - \sigma_k \bfY) \bfa_k}_2^2 \\
            &= \sum_{k = 1}^\ntemplates \sum_{i = 1}^\nrows \abs{(\bfx_i - \bfy_{\sigma_k(i)})^\top \bfa_k}^2 = \sum_{k = 1}^\ntemplates \left( \sum_{i = 1}^{i_k-1} \abs{(\bfx_i - \bfx_{i+1})^\top \bfa_k}^2 + \abs{\bfx_{i_k}^\top \bfa_k}^2 \right) \\
            &\leq \sum_{k = 1}^\ntemplates \norm{\bfa_k'}_2^2 \left( \sum_{i = 1}^{i_k-1} \norm{\bfx_i - \bfx_{i+1}}_2^2 + \frac{\pi^2}{\nrows^2} \right) \leq \sum_{k = 1}^\ntemplates \norm{\bfa_k'}_2^2 \left( \frac{4 \pi^2 (i_k-1)}{\nrows^2} + \frac{\pi^2}{\nrows^2} \right) \\
            &\leq \frac{\pi^2}{\nrows}\left( 2 + \frac{1}{\nrows} \right) \left( \sigma_{\ncolumns-1}^2 + \sigma_{\ncolumns}^2 \right)
        \end{align*}
        and a similar estimate shows the same for the case $i_k > \nrows / 2 + 1$.

        Finally, since the upper Lipschitz constant of $\QuotientSortEmbedding{\bfA}$ is given by the largest singular value $\sigma_1$ of $\bfA$, it follows that the distortion must be in $\Omega(\nrows^{1/2})$.
    \end{proof}

    \begin{remark}\label{rem:switch_of_quantors}
        In the above proof, we choose $\bfX, \bfY \in \bbR^{n \times d}$ \emph{depending on $\bfA \in \bbR^{d \times D}$} in order to obtain a bound on the lower Lipschitz constant of $\QuotientSortEmbedding{\bfA}$ that depends on the two smallest singular values, $\sigma_{d-1}$ and $\sigma_d$, of $\bfA$. Alternatively, we might as well let $\bfX, \bfY \in \bbR^{n \times d}$ have rows
        \begin{equation*}
            \bfx_i := \begin{pmatrix}
                \bfnull_{1 \times (\ncolumns-2)} & \cos( 2 \pi i / \nrows ) & \sin( 2 \pi i / \nrows )
            \end{pmatrix}
        \end{equation*}
        and $\bfy_1 := \bfnull_{1 \times \ncolumns}$ as well as $\bfy_i := \bfx_i$ independent of $\bfA$ (i.e., without assuming that the rows of $\bfA$ correspond to its singular values multiplied by its right singular vectors). In this way, we obtain the slightly worse upper bound
        \begin{equation*}
            \frac{(2+1/n)^{1/2}\pi}{\nrows^{1/2}} \cdot \left( \sigma_1^2 + \sigma_2^2 \right)^{1/2}
        \end{equation*}
        for the lower Lipschitz constant. The benefit of this approach is, of course, that it is completely independent of $\bfA$. In particular, this shows that there exist matrices $\bfX, \bfY \in \bbR^{n \times d}$ such that, for all $\bfA \in \bbR^{d \times D}$, it holds that 
        \begin{equation}\label{eq:I_cannot_think_of_anything}
            \norm{\SortEmbedding{\bfA}(\bfX) - \SortEmbedding{\bfA}(\bfY)}_\mathrm{F}^2 \lesssim \frac{\sigma_1^2 + \sigma_2^2}{n} \cdot \distance{\bfX}{\bfY}^2.
        \end{equation}
    \end{remark}

    We have presented three settings in which the distortion is in $O(\nrows^2)$ (or in $\smash{\widetilde O(\nrows^2)}$) and we have shown that the distortion is always in $\Omega(\nrows^{1/2})$. This leaves a slight gap and it would be interesting to understand whether the lower bound is tight; i.e., whether one can construct a matrix $\bfA \in \bbR^{\ncolumns \times \ntemplates}$ such that the distortion of $\QuotientSortEmbedding{\bfA}$ is in $\widetilde O(\nrows^{1/2})$ or even in $O(\nrows^{1/2})$.

\subsection{Bi-Lipschitz Bounds for \texorpdfstring{$\beta_{\bfA,\linop}$}{beta} }
\label{ssec:BiLipschitzBoundsBetaAL}
 
The  results in our previous sections, which guarantee bi-Lipschitzness, require a higher embedding dimension than what is required for injectivity only. For example, for injectivity we know that we can choose $D \sim nd$, but to get a bound of $\sim n^2$ on the bi-Lipschitz distortion in Theorem \ref{thm:DistortionUnitSphere} we needed $D\sim n^2d $.  In this subsection, we claim that the mapping $\QuotientSortEmbedding{\bfA,\linop}=\linop\circ \QuotientSortEmbedding{\bfA}$ obtained by applying a dimension reduction linear map $\linop$ to $\QuotientSortEmbedding{\bfA}$, will have similar distortion as $\QuotientSortEmbedding{\bfA}$ with an embedding dimension which is proportional to $nd$. 
\begin{theorem}\label{th-betaA,S}
	Let $\epsilon,\eta \in(0,1)$ and let $n,d,D\geq 2$ be natural numbers. Let $\bfA\in \RR^{d\times D}$ such that $\QuotientSortEmbedding{\bfA}$ is bi-Lipschitz with lower and upper Lipschitz constants $C_1$ and $C_2$, respectively. Then, for natural 
$$M=O\left(\epsilon^{-2}(nd\log(1/\epsilon)+\log(1/\eta)+nd\log(Dn^2))\right),$$
	we have that with probability of at least $1-\eta$, the function $\QuotientSortEmbedding{\bfA,\linopmat}=\linopmat  \operatorname{vec}(\QuotientSortEmbedding{\bfA})$ defined by 
	a matrix $\linopmat \in \RR^{M\times (nD)}$ whose entries are drawn independently from $\smash{\mathcal{N}(0,\frac{1}{\sqrt{M}})}$, will have a lower Lipschitz constant lower bounded by  $(1-\epsilon)C_1$ and upper bounded by $(1+\epsilon)C_2 $. Here, $\operatorname{vec}:\bbR^{n\times D}\rightarrow\bbR^{nD}$ denotes the flattening map.
\end{theorem}

\begin{proof}
	We begin with the following lemma
	\begin{lemma}
There is a finite number of linear transformations
\(\calA_1,\ldots,\calA_r:\RR^{2dn}\to \RR^{n\times D}\), where
\[
r=r(n,d,D)\leq (n^2D)^{2nd},
\]
such that, for all \((\bfX,\bfY)\in \RR^{2nd}\), there exists some index
\(t(\bfX,\bfY)\in [r]\) such that
\begin{equation}\label{eq:PWL}
\beta_\bfA(\bfX)-\beta_\bfA(\bfY)=\calA_t(\bfX,\bfY).
\end{equation}
\end{lemma}

	\begin{proof}
		In this proof, we will identify the space of matrices $(\bfX,\bfY)\in \RR^{n\times d}\oplus \RR^{n \times d}  $ with $\RR^{2nd} $.	
		
		We consider for all $k\in D$ and $i,j\in[n]\times [n]$, where $i \not= j$, the hyperplanes 
		\begin{align*}& H_{i,j,k}^{(1)}=\{(\bfX,\bfY)\in \RR^{2nd}\,|\, \bfx_i^T\bfa_k=\bfx_j^T\bfa_k  \}, \\
			&  H_{i,j,k}^{(2)}=\{(\bfX,\bfY)\in \RR^{2nd}\,| \, \bfy_i^T\bfa_k=\bfy_j^T\bfa_k  \} \end{align*}
		This gives us a  collection of 
$$H(n,d,D)=2D\cdot \left( \begin{array}{c} n \\ 2 \end{array}  \right)
=D(n^2-n) $$
		hyperplanes, defined in a vector space of dimension $T(n,d)=2nd $. From the theory of hyperplane arrangement \cite{zaslavsky1975facing,hypernotes}, we know that
		\begin{equation}\label{eq:complement}
			\RR^{2nd}\setminus \bigcup_{1\leq i<j \leq n, k\in [D], \ell\in \{1,2\}} H_{i,j,k}^{(\ell)} 
		\end{equation}
		can be written as a finite union of $r$ disjoint open convex polyhedra, where 
$$r\leq 1+H+
\left( \begin{array}{c} H \\ 2 \end{array}  \right)
+\ldots+ 
\left( \begin{array}{c} H \\ T \end{array}  \right)
. $$
		It can be easily shown by induction that, if $H,T\geq 2$, then this expression is bounded by 
$$r\leq 1+H+ \left( \begin{array}{c} H \\ 2 \end{array}  \right)
+\ldots+ \left( \begin{array}{c} H \\ T \end{array}  \right)
\leq H^T, $$
		which for our value of $T(n,d)$ and $H(n,d,D) $ gives us 
		$$r(n,d,D)\leq (Dn^2)^{2nd}  $$
		disconnected open polyhedra $\calP_1,\ldots,\calP_r$. We claim that, for each such polyhedron $\calP_t $, there corresponds a unique $\calA_t $ satisfying \eqref{eq:PWL} for all $(\bfX,\bfY)\in \calP_t$. To see this, fix some such $(\bfX,\bfY)$. Then, there exist $D $ permutation matrices $\bfP[k,X]$, $k\in [D]$ and $D$ permutation matrices $\bfP[k,Y]$, $k\in [D]$ , such that for  $k \in [D]$ the  $k$-th column of $\beta_\bfA(\bfX)-\beta_\bfA(\bfY)$ is given by
		\begin{equation*}
			[\beta_A(\bfX)-\beta_A(\bfY)]_{*,k} = \sort{\bfX \bfa_k}-\sort{\bfY \bfa_k} = \bfP[k,\bfX]\bfX \bfa_k-\bfP[k,\bfY]\bfY \bfa_k.
		\end{equation*}
		We now claim that, if $(\bfX,\bfY)$ and $(\hat\bfX,\hat \bfY)$ belong to the same polytope ${\cal P}_t $, then 
	\begin{equation}\label{eq:sameP}\bfP[k,\bfX]=\bfP[k,\hat \bfX], \qquad \mbox{for all } k\in [D] .\end{equation}
		Otherwise, there would have to be some $k\in [D]$ and $1\leq i<j \leq n$ such that 
		\[x_i^T a_k - x_j^T a_k < 0 < \hat x_i^T a_k - \hat x_j^T a_k\]
		This would imply, that on the straight line between $\bfX$ and $\hat \bfX$ there is some point $\tilde \bfX $ for which $\tilde \bfx_i^T\bfa_k-\tilde \bfx_j^T\bfa_k=0 $. But $\tilde \bfX$ would also be in the polyhedron $\calP_t$ since it is convex, which would mean that $\calP_t$ instersects the hyperplane $\smash{H_{i,j,k}^{(1)}}$ which  is a contradiction. Thus we have proven \eqref{eq:sameP}, and a similar argument also shows that
		$$\bfP[k,\bfY]=\bfP[k,\hat \bfY] . $$
		Accordingly,  for $k\in [D]$, $t\in [r]$ we define $\bfP[k,t,1] $ and $\bfP[k,t,2] $  to be the permutations satisfying
		\[
            \bfP[k,\bfX]=\bfP[k,t,1], \qquad \bfP[k,\bfY]=\bfP[k,t,2],
        \]
        for all $(\bfX,\bfY)\in \RR^{2nd}$, and we define $\calA_t:\RR^{2nd}\to \RR^{n\times D} $ to be the linear mapping whose $k$-th column is given by 
		$$\left[\calA_t(\bfX,\bfY) \right]_{*,k}=\bfP[k,t,1]\bfX \bfa_k-\bfP[k,t,2]\bfY \bfa_k.$$
		From what we saw, we know that $\calA_t(\bfX,\bfY)=\beta_A(\bfX)-\beta_A(\bfY) $ for all $(\bfX,\bfY)\in \calP_t$. Thus, we know that \eqref{eq:PWL} holds with at most $r$ different linear transformations, at least for all $(\bfX,\bfY)$ in the complement of the hyperplanes we defined. The fact that \eqref{eq:PWL} holds also for $(\bfX,\bfY)$ belonging to one of the hyperplanes follows from a continuity argument.
	\end{proof}

	To conclude the proof of the theorem \ref{th-betaA,S}, we will use some known results from the field of sketching algorithms, see e.g., \cite{robi,mit}.
	
	A random matrix $\linopmat \in \mathbb{R}^{M\times N}$ is called an $(\epsilon, \delta, k)$-Oblivious Subspace Embedding (OSE) if, for all linear $\mathcal A : \mathbb R^k \to \mathbb R^N$,
	\begin{equation*}
	    \mathbb{P}_\linopmat \lbrace \forall \mathbf x \in \mathbb{R}^k,~ \|\linopmat \calA \mathbf x\| \in  (1 \pm \epsilon)\|\calA \mathbf x\| \rbrace \geq  1-\delta.
	\end{equation*}
	It is known that if $M=O(\epsilon^{-2}(k \log(1/\epsilon)+\log(1/\delta))$ and the entries of $\linopmat\in \RR^{M\times N}$ are drawn independently from a normal distribution scaled by $\smash{\tfrac{1}{\sqrt{M}}}$, then $\linopmat$ is a $(\epsilon, \delta, k)$-Oblivious Subspace Embedding. 
	
	Using a simple union bound, we can extend this to the case of $r$ different linear maps, namely, for all linear $\calA_1,\ldots,\calA_r : \mathbb R^k \to \mathbb R^N$
	\begin{equation} \label{eq:As}
		 \mathbb{P}_\linopmat \lbrace \forall \mathbf x \in \mathbb{R}^k, \forall j\in [r]
		,~ \|\linopmat \calA_j \mathbf x\| \in  (1 \pm \epsilon)\|\mathcal A_j \mathbf x\| \rbrace \geq  1-r\delta.
	\end{equation}
	To conclude the proof of the theorem, we use this result, setting $k=2nd$, $N=nD$, $r=r(n,d,D) \leq (Dn^2)^{2nd}$, $\delta=\frac{\eta}{r}$, and obtain that, for  
	\begin{align*}M &= O(\epsilon^{-2}(k \log(1/\epsilon)+\log(1/\delta))\\
		&=O(\epsilon^{-2}(2nd \log(1/\epsilon)+\log(1/\eta)+\log((Dn^2)^{2nd}))\\
		&=O\left(\epsilon^{-2}\left(\log(1/\eta)+nd (\log(1/\epsilon)+\log(Dn^2)) \right) \right),
	\end{align*}
	we have that the matrix $\linopmat$ satisfies \eqref{eq:As} with probability $\geq 1-r\delta=1-\eta$ for the collection of $\calA_1,\ldots,\calA_r$ described in the lemma. Therefore, for any fixed  $\bfX,\bfY \in \RR^{d\times n}$, there is an appropriate $t\in [r]$ such that $\beta_\bfA(\bfX)-\beta_\bfA(\bfY)=\calA_t(\bfX,\bfY), $ and then
	\begin{align*}
		\| \beta_{\bfA,\linopmat}(\bfX)-\beta_{\bfA,\linopmat}(\bfY)\|_2 &=\|\linopmat\left( \beta_{\bfA}(\bfX)-\beta_{\bfA}(\bfY) \right)\|_2 =\|\linopmat\left( \calA_t(\bfX,\bfY) \right)\|_2\\
		&\geq (1-\epsilon)\|\calA_t(\bfX,\bfY)\|_2 = (1-\epsilon) \|\beta_{\bfA}(\bfX)-\beta_{\bfA}(\bfY)\|_F\\
		&\geq (1-\epsilon)C_1    \distance{\bfX}{\bfY} 
	\end{align*}
	Similarly, we can show that 
	$$ \| \beta_{\bfA,\linopmat}(\bfX)-\beta_{\bfA,\linopmat}(\bfY)\|_2\leq  (1+\epsilon)C_2    \distance{\bfX}{\bfY}$$
	which concludes the proof.
\end{proof}

    \section{Numerical Results}\label{sec:NumericalResults}
    We conclude with some numerical experiments looking into the optimal embedding dimension of $\beta_\bfA$. For small dimensions $n$ and $d$, we might use \cite[Proposition~3.8 on p.~14]{Balan2025Permutation} to analyse whether our results (Theorem~\ref{thm:orbit_separation_beta} and \ref{thm:matousek2}) are tight. The set of matrices $\bfX \in \bbR^{\ncolumns \times \nrows}$ at which $\SortEmbedding{\bfA}$ is orbit separating, that is, at which $\SortEmbedding{\bfA}(\bfX) = \SortEmbedding{\bfA}(\bfY)$ implies $\bfX \sim_{S_n} \bfY$ for all $\bfY \in \bbR^{\ncolumns \times \nrows}$, is completely characterized for fixed $\bfA = (\boldsymbol{I}_d | \bfa_1 ~ \dots ~ \bfa_{D-d}) \in \bbR^{\ncolumns \times \ntemplates}$: indeed, $\SortEmbedding{\bfA}$ is \emph{not} orbit separating at $\bfX \in \bbR^{\ncolumns \times \nrows}$ if and only if there exist $(\bfP_i)_{i = 1}^d \in S_n$, $(\bfQ_j)_{j = 1}^{D-d} \in S_n$ such that 
    \begin{gather*}
        \forall j \in [D-d],~ \begin{pmatrix}
            (\bfP_1 - \bfQ_j) \bfx_1 & \dots & (\bfP_d - \bfQ_j) \bfx_d
        \end{pmatrix} \bfa_j = 0, \\
        \text{and}\\
        \forall \bfP \in S_n, \exists i \in [d] : (\bfP - \bfP_i) \bfx_i \neq \bfnull_n.
    \end{gather*}
    
    The conditions above can be implemented so that we may simply check whether a given $\bfA = (\boldsymbol{I}_d | \bfa_1 ~ \dots ~ \bfa_{D-d}) \in \bbR^{\ncolumns \times \ntemplates}$ is such that $\SortEmbedding{\bfA}$ separates orbits. Applying this idea to matrices $\bfA$ whose last $D-d$ columns are randomly generated, allows us to conclude that, in the following cases, $\SortEmbedding{\bfA}$ separates orbits: 

    \begin{itemize}
        \item $n = 3$, $d = 3$, $D = 6$, \[\bfA = \begin{pmatrix}
            1 & 0 & 0 & 0.56 & 0.66 & 0.21 \\
            0 & 1 & 0 & 0.24 & 0.58 & 0.81 \\
            0 & 0 & 1 & 0.71 & 0.53 & 0.45
        \end{pmatrix}\]
        \item $n = 3$, $d = 4$, $D = 8$, \[\bfA = \begin{pmatrix}
            1 & 0 & 0 & 0 & 0.32 & 0.38 & 0.49 & 0.75 \\
            0 & 1 & 0 & 0 & 0.95 & 0.77 & 0.45 & 0.28 \\
            0 & 0 & 1 & 0 & 0.03 & 0.80 & 0.65 & 0.68 \\
            0 & 0 & 0 & 1 & 0.44 & 0.19 & 0.71 & 0.66
        \end{pmatrix}\]
        \item $n = 4$, $d = 2$, $D = 4$, \[\bfA = \begin{pmatrix}
            1 & 0 & 0.83 & 0.16 \\
            0 & 1 & 0.95 & 0.78
        \end{pmatrix}\]
        \item $n = 5$, $d = 2$, $D = 5$, \[\bfA = \begin{pmatrix}
            1 & 0 & 0.814724 & 0.126987 & 0.632359 \\
            0 & 1 & 0.905792 & 0.913376 & 0.097540
        \end{pmatrix}\]
    \end{itemize}
    In several cases, our implementation produced matrices 
    $\bfA$ for which $\beta_\bfA$ does \emph{not} separate orbits. This might suggest that in these cases orbit separation fails generically. Concretely, randomly generated matrices did not produce orbit generating embeddings when:
    \begin{multicols}{2}
        \begin{itemize}
            \item $n = 3$, $d = 2$, $D = 3$ 
            \item $n = 3$, $d = 3$, $D = 5$ 
            \item $n = 3$, $d = 4$, $D = 7$ 
            \item $n = 6$, $d = 2$, $D = 5$ 
        \end{itemize}
    \end{multicols}
    \noindent
    We have not considered higher dimensional cases because our implementation becomes numerically intractable once $n$ or $d$ are large.

    \begin{table}
        \centering
        \begin{subtable}[t]{0.54\textwidth}
            \centering
            \begin{tabular}{ c | c c c c c }
                 $n \backslash d$ & $2$ & $3$ & $4$ & $5$ & $6$ \\ \hline
                 $2$ & $\boldsymbol{6}$ & $\boldsymbol{10}$ & $\boldsymbol{14}$ & $\boldsymbol{18}$ & $\boldsymbol{22}$ \\
                 $3$ & $12$ & $21^{\scriptscriptstyle (18)}$ & $30^{\scriptscriptstyle (24)}$ & $39$ & $48$ \\
                 $4$ & $20^{\scriptscriptstyle (16)}$ & $36$ & $52$ & $68$ & $84$ \\
                 $5$ & $30^{\scriptscriptstyle (25)}$ & $55$ & $80$ & $105$ & $130$ \\
                 $6$ & $42$ & $78$ & $114$ & $150$ & $186$
            \end{tabular}
            \caption{Minimal embedding dimension $nD$ for which our result, Theorem~\ref{thm:orbit_separation_beta}, guarantees that $\beta_\bfA$ separates orbits (with full spark $\bfA$).}\label{ta:Minimal}
        \end{subtable}
        ~
        \begin{subtable}[t]{0.42\textwidth}
            \centering
            \begin{tabular}{ c | c c c c c }
                 $n \backslash d$ & $2$ & $3$ & $4$ & $5$ & $6$ \\ \hline
                 $2$ & $\boldsymbol{4}$ & $\boldsymbol{8}$ & $\boldsymbol{12}$ & $\boldsymbol{16}$ & $\boldsymbol{20}$ \\
                 $3$ & $6$ & $12$ & $18$ & $24$ & $30$ \\
                 $4$ & $\boldsymbol{12}$ & $24$ & $36$ & $48$ & $60$ \\
                 $5$ & $15$ & $30$ & $45$ & $60$ & $75$ \\
                 $6$ & $18$ & $36$ & $54$ & $72$ & $90$
            \end{tabular}
            \caption{Maximal embedding dimension $nD$ for which our result, Theorem~\ref{thm:matousek2}, shows that $\beta_\bfA$ does not separate orbits (independently of the choice of $\bfA$).}\label{ta:Maximal}
        \end{subtable}
        \caption{Entries in which our results are \textbf{optimal} (i.e., yield the smallest possible $\ntemplates \in \bbN$ for which there exists an $\bfA \in \bbR^{\ncolumns \times \ntemplates}$ such that $\SortEmbedding{\bfA}$ separates orbits/yield the largest possible $\ntemplates$ for which $\SortEmbedding{\bfA}$ does not separate orbits independently of the choice of $\bfA$) are highlighted in \textbf{bold}. Entries for which we know that our results are suboptimal are decorated with a dimension for which we were able to find a orbit separating embedding in brackets. All dimensions for which it is not known whether our result is optimal have no special styling.}
        \label{ta:MinimalMaximal}
    \end{table}

    We summarize our current knowledge, consisting of Theorems~\ref{thm:orbit_separation_beta} and \ref{thm:matousek2} as well as the above results, in two tables. Table~\ref{ta:Minimal} records the minimal embedding dimension $nD$ for which orbit separation is guaranteed while Table~\ref{ta:Maximal} records the maximal embedding dimension $nD$ for which orbit separation is ruled out independently of $\bfA$ (see also Figure~\ref{fig:Gaps} for a visualization when $d=2$). The code used to generate the examples in this section is publicly available on GitHub at \href{https://github.com/rvbalan/SortingBasedUniversalKeys}{\texttt{rvbalan/SortingBasedUniversalKeys}}.

    \section{Conclusions}\label{sec:Conclusion}

In this paper we studied bi-Lipschitz embeddings of the quotient space $\bbR^{n\times d}/{\sim}$, where the equivalence is induced by the
action $\mathbf X\mapsto \mathbf P \mathbf X$ of the permutation group $S_n$. We discussed three $S_n$-invariant embeddings $\beta_\bfA$, $\beta_{\bfA,\linop}$, 
and  $\delta_{\bfA,\bfB}$, constructed via linear mappings and sorting operators.

We demonstrated that injective embeddings are achievable with relatively low embedding dimensions: 
as low as $n^2(d-1)+n$ for $\beta_\bfA$, and as low as $2nd-d$ 
for $\beta_{\bfA,\linop}$ and $\delta_{\bfA,\bfB}$. 

We then analyzed the bi-Lipschitz distortion of these embeddings.
When $D\sim n^2d$, the map $\QuotientSortEmbedding{\bfA}$ achieves  distortion scaling as $O(n^2)$ up to logarithmic factors, independent of $d$. Moreover,
$\QuotientSortEmbedding{\bfA,\linop}$ can attain comparable bi-Lipschitz distortion, 
provided the embedding dimension scales proportionally to $nd$, up to logarithmic factors. On the other hand, we show that the distortion of $\QuotientSortEmbedding{\bfA}$ cannot be better than $\sqrt{n}$.

Many interesting open questions remain. Firstly, there is a gap between the best $\sim n^2$ distortion we can achieve and the $\sqrt n$ lower bound on the distortion, and it will be interesting to close this gap and definitely find the optimal distortion attainable by a mapping of the form $\QuotientSortEmbedding{\bfA}$. Secondly, our results focus on the metric obtained by quotienting the Frobenius norm over the permutation group, and it could be interesting to understand the distortion with respect to other norms. Thirdly, it could be interesting to understand the bi-Lipschitz distortion of other permutation-invariant embeddings, like the max-filtering \cite{Cahill2024GroupInvariant} or FSW \cite{Amir2024Fourier} embeddings. Finally, while some works have tried to establish the advantage of sorting-based embeddings and other bi-Lipschitz embeddings in machine learning tasks \cite{Balan2025Permutation,Davidson2024Hoelder,sverdlov2024fsw}, less expressive pooling mechanisms are still much more prevalent. Empirically establishing cases where bi-Lipschitz embeddings are crucial for high performance is thus an important experimental goal.

    
    
    

\section*{Acknowledgments}
The authors acknowledge the use of OpenAI's ChatGPT to assist with phrasing and typesetting suggestions. N.D. has been supported in part by ISF grant 272/23. 
R.B.~has been supported in part by the National Science Foundation under grant NSF DMS-2510216. 
We thank C.~Lester and R.~Sørensen \cite{Lester2026How} for pointing out two typographical errors in Section~\ref{sec:NumericalResults}.

    \bibliography{lit}
    \bibliographystyle{alpha}

\appendix

\section{Background on Real Algebraic Geometry}\label{app:semialgebraic}
    
    A subset $\calS \subset \bbR^n$ is \emph{semialgebraic} if it can be constructed from building blocks of the form
    \begin{equation*}
        \set{ x \in \bbR^n }{ p(x) = 0 }, \qquad \set{ x \in \bbR^n }{ p(x) > 0 }
    \end{equation*}
    by taking finite unions, intersections and complements, where $p$ is a real-valued polynomial in $n$ variables. Similarly, a function $f : \calS \subset \bbR^n \to \bbR^m$ is \emph{semialgebraic} if its graph,
    \begin{equation*}
        \mathrm{Graph}(f) := \set{ (x,f(x)) \in \bbR^{n + m} }{ x \in \calS },
    \end{equation*}
    is semialgebraic. Given two semialgebraic sets $\calS \subset \bbR^n$ and $\mathcal{T} \subset \bbR^m$, a \emph{(semialgebraic) homeomorphism} is a bijective continuous semialgebraic map $f : \calS \to \mathcal{T}$ with continuous semialgebraic inverse. If a semialgebraic homeomorphism exists between semialgebraic sets $\calS$ and $\mathcal{T}$, we call them \emph{(semialgebraically) homeomorphic}.
    
    Semialgebraic sets are known to decompose in the following way. 

    \begin{theorem}[{\cite[Theorem~2.3.6 on p.~33]{Bochnak1998Real}}]
        Every semialgebraic subset of $\bbR^n$ is the disjoint union of a finite number of semialgebraic sets, each of them (semialgebraically) homeomorphic to an open hypercube $(0,1)^d$, for some $d \in \bbN$ (with $(0,1)^0$ being a point).
    \end{theorem}

    Consider a semialgebraic set $\calS \subset \bbR^n$ which is the finite union of semialgebraic sets homeomorphic to hypercubes of dimensions $(d_i)_{i = 1}^p \in \bbN$. Then, the \emph{(semialgebraic) dimension} of $\calS$ is $\max_{i \in [p]} d_i$.
    
    Finally, we note that, if $\calS \subset \bbR^n$ and $\mathcal{T} \subset \bbR^m$ are two semialgebraic sets and $f : \calS \times \mathcal{T} \to \bbR$ is a semialgebraic function, then all sets of the form 
    \begin{equation*}
        \set{ y \in \mathcal{T} }{ f(x,y) = 0 }, \qquad x \in \calS,
    \end{equation*}
    are semialgebraic as well: indeed, the above set is the image of $\mathrm{Graph}(f) \cap ( \lbrace x \rbrace \times \mathcal{T} \times \lbrace 0 \rbrace )$ by the projection $\calS \times \mathcal{T} \times \bbR \to \mathcal{T}$ and semialgebraic sets are stable under projections \cite[Theorem~2.2.1 on p.~26]{Bochnak1998Real}.

\section{Proof of Theorem \texorpdfstring{\ref{thm:biLipschitzdeterministic}}{2}}
   Let $\bfX, \bfY \in \bbR^{\nrows \times \ncolumns}$ be arbitrary but fixed with rows $(\bfx_i)_{i = 1}^\nrows, (\bfy_i)_{i = 1}^\nrows$, respectively, and let $(\bfa_k)_{k = 1}^\ntemplates$ denote the columns of $\bfA \in \bbR^{\ncolumns \times \ntemplates}$. There exist permutations $(\sigma_k)_{k = 1}^\ntemplates \in S_\nrows$ and associated permutation matrices $(\boldsymbol{\Pi}_k)_{k=1}^D$ such that 
        \begin{align*}
            \norm{\SortEmbedding{\bfA}(\bfX) - \SortEmbedding{\bfA}(\bfY)}_\mathrm{F}^2 &= \sum_{k = 1}^\ntemplates \norm{\sort{\bfX \bfa_k} - \sort{\bfY \bfa_k}}_2^2 = \sum_{k = 1}^\ntemplates \norm{\bfX \bfa_k - \boldsymbol{\Pi}_k \bfY \bfa_k}_2^2 \\
            &= \sum_{i = 1}^\nrows \sum_{k = 1}^\ntemplates \abs{(\bfx_i - \bfy_{\sigma_k(i)})^\top \bfa_k}^2 = \sum_{i,j = 1}^\nrows \sum_{k \in I_{i,j}} \abs{(\bfx_i - \bfy_j)^\top \bfa_k}^2,
        \end{align*}
        where $I_{i,j} := \set{k \in [\ntemplates]}{ \sigma_k(i) = j }$.

        Consider the following trick: we observe that the matrix $\bfS \in \bbR^{\nrows \times \nrows}$ given by
        \begin{equation}\label{eq:DS}
            S_{i,j} := \frac{\abs{I_{i,j}}}{\ntemplates} 
        \end{equation}
        is doubly stochastic. As  such,  it can be written as the convex combination of permutation matrices, due to a classical result of Birkhoff \cite{Birkhoff1946Three} and von Neumann \cite{vonNeumann1953A}. In fact, the polytope of doubly stochastic matrices has dimension $(\nrows-1)^2$, and thus  Carathéodory's theorem (cf.~e.g.~\cite{Gruenbaum2003Convex}) implies that we can write $\bfS$ as a convex combination of $N=(n-1)^2+1 $ permutation matrices, namely
        \begin{equation*}
            \bfS = \sum_{\ell = 1}^\BirkhoffDimension t_\ell \bfP^{(\ell)}, 
        \end{equation*}
        where the $t_\ell$ are nonnegative numbers with $\sum_{\ell = 1}^\BirkhoffDimension t_\ell = 1$, and the  $\bfP^{(\ell)}$ are permutation matrices. It follows that (at least) one of the coefficients $k$ out of $N$ satisfies $t_k\geq 1/N$. Let $\sigma$ be the permutation  for which $\bfP^{(k)}_{i,\sigma(i)}=1 $ for all $i\in [n]$. Then, 
        \begin{equation*}
            \bfS_{i,\sigma(i)}= \sum_{\ell = 1}^\BirkhoffDimension t_\ell \bfP^{(\ell)}_{i,\sigma(i)}\geq t_k\bfP^{(k)}_{i,\sigma(i)}=t_k\geq \frac{1}{N}, \qquad i\in [n].
        \end{equation*}
       This result, together with the definition of $\bfS$ in \eqref{eq:DS}, implies that $I_{i,\sigma(i)}$ has cardinality greater than or equal to $\ntemplates/\BirkhoffDimension \geq r \ncolumns$.

        Going back to our initial computation and letting $I_i \subset I_{i,\sigma(i)}$ be an arbitrary subset of cardinality $r \ncolumns$, we conclude that

                \begin{align*}
            \MoveEqLeft[3] \norm{\SortEmbedding{\bfA}(\bfX) - \SortEmbedding{\bfA}(\bfY)}_\mathrm{F}^2 
        	= \sum_{i,j = 1}^\nrows \sum_{k \in I_{i,j}} \abs{(\bfx_i - \bfy_j)^\top \bfa_k}^2 \geq \sum_{i = 1}^\nrows \sum_{k \in I_i} \abs{(\bfx_i - \bfy_{\sigma(i)})^\top \bfa_k}^2 \\
        	={}& \sum_{i = 1}^\nrows \norm{(\bfx_i - \bfy_{\sigma(i)})^\top \bfA(I_i)}_2^2 \geq \sum_{i = 1}^\nrows \sigma_\ncolumns^2(\bfA(I_i)) \norm{\bfx_i - \bfy_{\sigma(i)}}_2^2 \\
        	\geq{}& \min_{\substack{I \subset [\ntemplates]\\\abs{I} = r\ncolumns}} \sigma_\ncolumns^2(\bfA(I)) \sum_{i = 1}^\nrows \norm{\bfx_i - \bfy_{\sigma(i)}}_2^2 = \min_{\substack{I \subset [\ntemplates]\\\abs{I} = r\ncolumns}} \sigma_\ncolumns^2(\bfA(I)) \norm{\bfX - \GroupAction{P}{\bfY}}_\mathrm{F}^2 \\
        	\geq{}& \min_{\substack{I \subset [\ntemplates]\\\abs{I} = r\ncolumns}} \sigma_\ncolumns^2(\bfA(I)) \cdot \distance{\bfX}{\bfY}^2,
        \end{align*}
        which finishes the proof.
\end{document}